\documentclass{article} 
\usepackage{dsfont}
\usepackage{enumerate}
\usepackage{mathrsfs}
\usepackage{nips15submit_e,times}
\usepackage{mdframed}
\usepackage{hyperref}
\usepackage{caption}
\usepackage{caption}
\usepackage{url}
\usepackage{amsthm}
\usepackage{amsmath}
\usepackage{amsfonts}
\usepackage{mathtools}
\usepackage[square,numbers, sort]{natbib}
\usepackage{pgfplots}
\usepackage{csquotes}
\usepackage[titletoc, title]{appendix}

\newcommand{\cinp}{\xrightarrow{p}}
\newcommand\given[1][]{\:#1\vert\:}
\newcommand{\sectionref}[1]{Section \ref{#1}}
\newcommand{\lemmaref}[1]{Lemma \ref{#1}}
\newcommand{\claimref}[1]{Claim \ref{#1}}
\newcommand{\theoremref}[1]{Theorem \ref{#1}}
\newcommand{\propositionref}[1]{Proposition \ref{#1}}
\newcommand{\exampleref}[1]{Example \ref{#1}}
\newcommand{\figureref}[1]{Figure \ref{#1}}

\providecommand{\abs}[1]{\lvert#1\rvert}

\newcommand{\ZC}{\mathcal{Z}_{\mathcal{C}}}
\DeclareMathOperator*{\argmax}{arg\,max}
\DeclareMathOperator*{\KL}{KL}
\DeclareMathOperator*{\DKLUCB}{{\text{\textnormal{{DKLUCB}}}}}
\DeclareMathOperator*{\KLUCB}{{\text{\textnormal{{KL-UCB}}}}}
\DeclareMathOperator*{\UCB}{{\text{\textnormal{{UCB}}}}}

\newcommand{\di}{\mathbin{/}}
\newcommand{\I}{{\mathds{1}}}
\newcommand{\E}{{\mathbb E}}
\newcommand{\F}{\mathcal{F}}
\newcommand{\X}{\mathcal{X}}
\newcommand{\B}{\mathcal{B}}
\newcommand{\N}{{\mathbb N}}
\newcommand{\R}{{\mathbb R}}
\newcommand{\A}{\mathcal{A}}
\newcommand{\K}{\mathcal{K}}
\renewcommand{\P}{\mathcal{P}}
\newcommand{\C}{\mathcal{C}}

\newcommand\numberthis{\addtocounter{equation}{1}\tag{\theequation}}
\numberwithin{equation}{section}
\newcommand{\why}[1]{\text{\big(#1\big)\ \ \ }}

\title{Regret vs. Communication: Distributed Stochastic Multi-Armed Bandits and Beyond}

\author{
    Shuang Liu\\
    Shanghai Jiao Tong University\\
    \texttt{liushuang93006@gmail.com} \\
    \And
    Cheng Chen\\
    Shanghai Jiao Tong University\\
    \texttt{jackchen1990@gmail.com} \\
    \And
    Zhihua Zhang \\
    Shanghai Jiao Tong University \\
    \texttt{zhang-zh@cs.sjtu.edu.cn} \\
}

\nipsfinalcopy 

\newtheorem{theorem}{Theorem}
\newtheorem{example}{Example}
\newtheorem{lemma}[theorem]{Lemma}
\newtheorem{claim}[theorem]{Claim}
\newtheorem{proposition}[theorem]{Proposition}
\newtheorem{definition}[theorem]{Definition}

\begin{document}

\maketitle

\begin{abstract}
    In this paper, we consider the distributed stochastic multi-armed bandit problem, where a global arm set can be accessed by multiple players independently. The players are allowed to exchange their history of observations with each other at specific points in time. We study the relationship between regret and communication. When the time horizon is known, we propose the Over-Exploration strategy, which only requires one-round communication and whose regret does not scale with the number of players. When the time horizon is unknown, we measure the frequency of communication through a new notion called the density of the communication set, and give an exact characterization of the interplay between regret and communication. Specifically, a lower bound  is established and stable strategies that match the lower bound are developed. The results and analyses in this paper are specific but can be translated into more general settings.
\end{abstract}

\section{Introduction}

We consider the distributed stochastic multi-armed bandit (MAB) problem, where a global arm set $\A = [K]$ can be accessed by $M$ players independently. Each arm $a\in\A$ is associated with an unknown yet fixed probability distribution $\nu_a$ that belongs to a known family $\P$. The process proceeds in rounds. At the beginning of each round $t$, each player $p\in[M]$ pulls an arm $A_{p, t}\in\A$ based on some \textit{policy} and independently receives a reward $X_{p, t}\sim\nu_{A_{p, t}}$. Some rounds are \textit{Communication rounds} at the end of which every player knows everything other players know. Note that when $M = 1$, the process reduces to a traditional MAB process.

The goal is to minimize the \textit{regret}. We denote by $\mu_a$ the mean of the distribution $\nu_a$ and let $\mu^* = \max_a\mu_a$. The regret after $T$ rounds is defined by 
\[
    \mathscr{R}_T \coloneqq TM\mu^* - \E\left[\sum_{t = 1}^{T}\sum_{p = 1}^{M}X_{p, t}\right].
\]
We are only interested in \textit{consistent} policies. A policy is said to be consistent if for any $c > 0$, $\mathscr{R}_T = o(T^c)$ always holds. If we further let $\Delta_a = \mu^* - \mu_a$ and $N_T(a)$ be the number of times arm $a$ has been pulled by all the players in the first $T$ rounds, then it suffices to bound $N_T(a)$ for all $a$ such that $\mu_a\neq\mu^*$ since the regret can be written as $\sum_{a\in\A, \mu_a\neq\mu^*}\Delta_a\E\left[N_T(a)\right].$


\subsection{Related Work}

The traditional single-player MAB problem has been studied for a long time. The establishment of the lower bounds can be traced back to \cite{LR, BK}, in which some asymptotically policies based on the notion of \textit{upper confidence bound} (UCB) are also developed. Later contributions mainly focus on finite-time analysis of UCB-like algorithms \cite{ACF, GC, CGMMS, MMS}. Recently, a Bayesian approach called Thompson sampling \cite{T} is also proved to be asymptotically optimal \cite{KKM}.

Although there have been regret analyses of single-player models with \textit{delayed feedback} \cite{JGS}, regret analysis of distributed MAB models remains largely unaddressed in the literature. The model used in \cite{HKKLS} is very similar to ours, but they focused on best arm identification, which is mainly an exploration problem. In \cite{SBHOJK}, a P2P-like gossip based model is proposed. However, their policy is based on the $\epsilon$-GREEDY algorithm \cite{ACF} which itself requires a lower bound on the gap between the best arm and the second-best arm. Furthermore, \cite{SBHOJK} only considered the case that the number of communication rounds grows linearly with $T$. There are also other distributed models in which players \textit{compete} with each other \cite{LZ} or an adversarial setting is considered \cite{AK}.

\section{Oblivious Bandit Policies}
\label{oblivious}

In this section we first propose a general framework for a large number of MAB policies. Then we show how this framework enables us to develop communication strategies independent of the bandit policy used. 
    \begin{definition}[Oblivious Bandit Policy]
        Given $K$ finite sets\footnote{Actually these sets are multisets} of rewards $\X_1, \X_2, \dotsc, \X_K$ generated by $K$ arms, an oblivious bandit policy chooses the next arm based only on these $K$ sets.
    \end{definition}
In order to explain how current bandit policies can be translated into this framework and further adapted into a distributed setting, we require some additional notation. Given a finite set of rewards $\X_a$ generated by arm $a$, the \textit{empirical distribution} with respect to $\X_a$ is defined by
    \[
        \hat{\nu}_{\X_a} \coloneqq \frac{1}{|\X_a|}\sum_{x\in \X_a}\delta(x),
    \]
    where $\delta(\cdot)$ is the Dirac delta function. For every player $p$, every round $t$, and every arm $a$, we denote by $\X_{p, t}(a)$ the set of rewards generated by arm $a$ that are available to player $p$ at the end of round $t$ and let $N_{p, t}(a) = |\X_{p, t}(a)|$. Clearly, if $t$ is a communication round, then $N_t(a) = N_{p,t}(a)$ for every $p$ and $a$. For simplicity, we use $\E[\nu]$ to denote the mean of distribution $\nu$ and denote $\E\left[\hat{\nu}_{\X_{p, t}(a)}\right]$ by $\tilde{\mu}_{p, t}(a)$. We also use $\B(p)$ to denote a Bernoulli distribution with parameter $p$, and $D_{\KL}(\nu_1||\nu_2)$ to denote the Kullback-Leibler (KL) divergence of $\nu_2$ from $\nu_1$. The following are two oblivious bandit policies that have been adapted into a distributed setting (and will be mainly discussed in this paper).

    \textbf{UCB adaptation for [0, 1] bounded rewards.} Let $\F(t) = \ln\left(t\ln^3(t)\right)$. $A_{p, t} = \argmax_{a} B_{p, t}^+(a)$ \[\text{where\ \ }B^+_{p, t}(a) =
        \tilde{\mu}_{p, t - 1}(a) + \sqrt{\frac{\F\left(\sum_{k = 1}^K N_{p, t -1}(k)\right)}{2N_{p, t - 1}(a)}}.
    \]

    \textbf{KL-UCB adaptation for Bernoulli rewards.} Let $\F(t) = \ln\left(t\ln^3(t)\right)$. $A_{p, t} = \argmax_{a} B_{p, t}^+(a)$ \[\text{where\ \ }B^+_{p, t}(a) = \sup\left\{p \in (0, 1): D_{\KL}\left(\hat{\nu}_{\X_{p, t - 1}(a)}||\B(p)\right) \leq \frac{\F\left(\sum_{k = 1}^K N_{p, t - 1}(k)\right) }{N_{p, t - 1}(a)}\right\}.
    \]
    There are many other oblivious bandit policies in the literature such as most UCB-like policies and Thompson sampling. These policies are called \textit{oblivious} because their choice of the next arm depends only on the empirical distribution of the data and the number of data collected. They do not rely on a \textit{timer} or other information such as the player's own previous choices or other players' previous choices. Note that there are non-oblivious policies such as the DMED policy \cite{HT}.

\section{Distributed Bandits: A Paradox}
\label{paradox}
    It is often believed that the more you know, the better you will do. Translating into bandit language, the more you communicate, the lower the regret is. However, the following example shows this is not necessarily true.
    \begin{example}
        \label{example}
        Suppose there are $2$ players, $2$ arms, and a total of $2^{16}$ rounds. Arm $1$ is associated with $\B(0.9)$ and arm $2$ is associated with $\B(0.8)$. Both players use the $\UCB$ adaptation described in \sectionref{oblivious}. Consider the following three communication strategies:
        \textbf{(A)} communicate once when $t = 2^{12}$;          
        \textbf{(B)} communicate three times when $t = 2^4, 2^8, 2^{12}$;       
        \textbf{(C)} communicate $2^{12}$ times when $t = 1, 2, \dotsc, 2^{12}$.
    \end{example}
    \begin{figure}
        \begin{center}
        \begin{tikzpicture}
        \begin{axis}
            [
height=4.0cm,width=10cm,scale only axis,
                xmode=log,
                xlabel={t ($\log$ scale)},
            ylabel={$N_t(2)$},
            xmin=10, xmax=100000,
            ymin=0, ymax=1000,
            xtick={100,1000,10000,100000},
            ytick={100,300,500, 700, 900},
            legend pos=north west,
            ymajorgrids=true,
            grid style=dashed,
            grid=both,
            legend entries = {No Communication, Full Communication, Strategy A, Strategy B, Strategy C}
        ]
         \node[coordinate,pin=above:{$t = 2^4$}]
            at (axis cs:16,50) {};
            \node[coordinate,pin=above:{$t = 2^{8}$}]
            at (axis cs:256,150) {};
            \node[coordinate,pin=above:{$t = 2^{12}$}]
            at (axis cs:4096,550) {};
        \addplot[
            color=magenta,
            mark=square,
            ]
            table{
                16 13.544900
                32 24.271900
                64 42.866200
                128 73.578500
                256 120.812300
                512 188.328800
                1024 274.390100
                2048 376.871400
                4096 490.605100
                8192 609.551900
                16384 729.795700
                32768 848.050800
                65536 961.369500
            };
        \addplot[
            color=teal,
            mark=o,
            ]
            table{
                16 12.643400
                32 21.832000
                64 36.544800
                128 59.060400
                256 90.660600
                512 131.626800
                1024 180.640800
                2048 233.885000
                4096 290.415800
                8192 347.490400
                16384 402.774400
                32768 456.110000
                65536 507.702400
            };
        \addplot[
            color=red,
            mark=triangle,
            ]
            table{
                16 13.559100
                32 24.268800
                64 42.932100
                128 73.678000
                256 120.783200
                512 187.878500
                1024 275.618600
                2048 378.265300
                4096 493.767500
                8192 493.777500
                16384 493.925800
                32768 496.746300
                65536 522.812800
            };
        \addplot[
            color=orange,
            mark=pentagon,
            ]
            table{
                16 13.671200
                32 22.442100
                64 38.407700
                128 67.398700
                256 113.230400
                512 132.182400
                1024 191.893100
                2048 282.991000
                4096 391.137300
                8192 393.183400
                16384 412.659500
                32768 483.830100
                65536 585.32150
            };
        \addplot[
            color=blue,
            mark=diamond,
            ]
            table{
                16 12.586400
                32 21.711600
                64 36.449800
                128 58.900600
                256 90.440400
                512 131.371200
                1024 180.154200
                2048 232.929200
                4096 288.695800
                8192 368.566000
                16384 466.199000
                32768 572.701600
                65536 679.60660
            };
        \end{axis}
        \end{tikzpicture}
    \end{center}
    \caption{Adding communication rounds may increase the regret. Strategy A communicates only one time ($t = 2^{12}$). Strategy B communicates three times ($t = 2^4, 2^8, 2^{12}$). Strategy C communicates over four thousand times ($t = 1, 2, \dotsc, 2^{12}$). Results are based on $10^4$ independent runs.}
    \label{exp1}
    \end{figure}
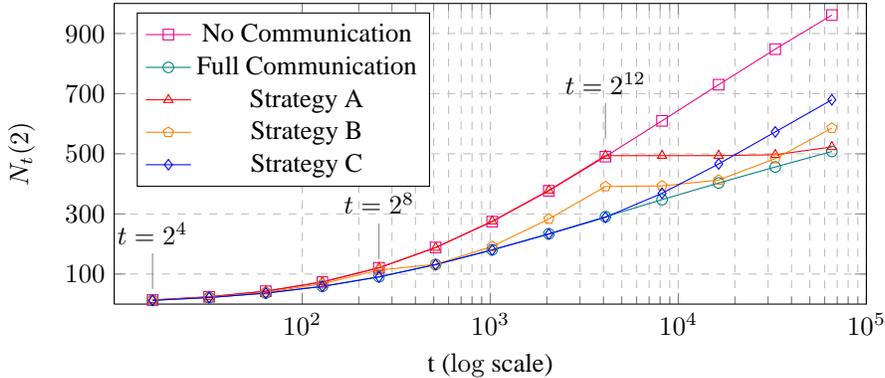
    Intuitively, Strategy A should be the worst and Strategy C should be the best. However, the numerical experiment gives a surprising result. \figureref{exp1} shows the average number of pulls of the suboptimal arm as a function of time (on a logarithmic scale) for Strategies A, B, C.\footnote{In this experiment, we use $\ln(2t)$ to approximate $\ln(\sum_{k = 1}^K N_{p, t - 1}(k)) + 3\ln\left(\ln\left(\sum_{k = 1}^K N_{p, t - 1}(k)\right)\right)$ for a better comparison when $T$ is relatively small.}

    To explain this phenomenon, first recall that almost all the single-player regret analysis \cite{ACF, AJO, CGMMS, GC, KKM, MMS} goes like follows:
    \begin{displayquote}
        Once a suboptimal arm has been pulled more than \ $\xi\ln(T)$ times, it  will \textit{almost} never be pulled any more before round $T$.
    \end{displayquote}
    
    Here $\xi$ is a constant depending on the bandit algorithm\footnote{e.g., $\frac{1}{2\Delta_a^2}$ in UCB and $\frac{1}{D_{\KL}\left(\B(\mu_a)||\B(\mu^*)\right)}$ in KL-UCB.}. Following this argument, we first consider Strategy C. It keeps communicating until $t = 2^{12}$. Then according to the full-communication curve, approximately $200$ additional explorations of the suboptimal arm are needed before $t = 2^{16}$. However, the two players have to do the $200$ explorations \textbf{separately} since they cannot communicate any more. Therefore this suboptimal arm is actually explored $2\cdot 200 = 400$ times from $t = 2^{12}$ to $t = 2^{16}$, resulting in a much higher final regret compared to the full-communication strategy, which only explores the suboptimal arm $200$ times during this time period. More generally, after a communication round, if an additional $\Delta$ of explorations of the suboptimal arm is needed before the next communication round, the actual explorations performed would be $M\cdot\Delta$. That is, due to lack of communication, $(M - 1)\cdot\Delta$ unnecessary explorations are performed, resulting in a larger regret.
    
    On the other hand, Strategy A cleverly makes $\Delta$ very close to $0$, forcing the final regret almost as low as that of the full-communication curve. When Strategy A finished its only communication, each of the players has already collected approximately 500 independent samples of the suboptimal arm, which is an over-exploration when $t = 2^{12}$ since the full-communication curve indicates that when $t = 2^{12}$, only approximately $300$ explorations of the suboptimal arm are needed. However, this amount of exploration happens to be just enough for $t = 2^{16}$ according to the full-communication curve. Therefore, when $t$ goes from $2^{12}$ to $2^{16}$, the suboptimal arm is rarely pulled.

    Finally, consider Strategy B. After each communication, the ``over-exploration'' phenomenon occurs, and the curve acts like that of Strategy A. After some period of time, the amount of explorations goes back to the normal level. Then it will act like Strategy C, resulting in an over-exploration again before the next communication comes. Strategy B performs communication in a relatively stable way. That is, it always keeps $\Delta$ relatively small, making its curve
fluctuates not so dramatically around the full-communication curve. 

Both Strategy A and Strategy B have implied some good ways to develop communication strategies. Strategy A indicates that we can find a time point such that the over-exploration is ``just enough." In \sectionref{one} we give the corresponding theoretical guarantees. However, this strategy requires a known time horizon, and the right chance for communication is sort of unpredictable for small $T$ and large $M$. Fortunately, Strategy B indicates that we can develop anytime strategies that require very few communication rounds with very good performance. The corresponding theoretical guarantees are given in \sectionref{dense}. However, if the communication rounds are too few, there is no way we can use an oblivious bandit policy to achieve a good performance. In \sectionref{sparse}, we introduce a non-oblivious bandit policy called DKLUCB (stands for Distributed KL-UCB) which is asymptotically optimal. Most of the results and analyses in \sectionref{one} and \sectionref{sustainable} suppose the KL-UCB adaptation described in \sectionref{oblivious} is used. Similar results also hold for the UCB adaptation (with a different constant) with only slight modification (even simplification) of our analyses. In particular, DKLUCB can be easily modified to get a UCB version called DUCB, which stands for Distributed UCB\footnote{As explained in \cite{CGMMS}, UCB is actually a simple relaxation of KL-UCB.}.

\section{The Over-Exploration Strategy}
\label{one}

In this section, we suppose that a time horizon $T$ is given. That is, the regret is only evaluated at the end of round $T$. As we have discussed in \sectionref{paradox}, we want to find the right chance for communication such that the resulting over-exploration is ``just enough." The following theorem tells us the right chance is around $T^{\frac{1}{M}}$.

\begin{theorem}
    \label{onethm}
    If the rewards are Bernoulli rewards and the only communication round is round $\lceil T^{\frac{1}{M}}\rceil$, then for every suboptimal arm $a$, 
    \begin{align*}
        \E\left[N_T^{\KLUCB}(a)\right] &\leq \frac{\ln(T)}{D_{\KL}\left(\B(\mu_a)||\B(\mu^*)\right)} + o\left(\ln(T)\right).
    \end{align*}
\end{theorem}

In other words, using the over-exploration strategy, asymptotically the regret does not scale with the number of players. Strangely, to prove an upper bound, we must prove a lower bound first. In fact, the following lemma is critical to our proof.

\begin{lemma}
    \label{lower1}
    If the rewards are Bernoulli rewards and the only communication round is round $\lceil T^{\frac{1}{M}}\rceil$, then for every suboptimal arm $a$ and any $\delta > 0$, 
    \begin{align*}
        \lim_{T\to\infty}\Pr\left(N^{\KLUCB}_{\left\lceil T^{\frac{1}{M}}\right\rceil}(a) \geq \frac{(1 - \delta)\ln(T)}{D_{\KL}\left(\B(\mu_a)||\B(\mu^*)\right)}\right) = 1.
    \end{align*}
\end{lemma}

This lemma ensures that we over-explore the suboptimal arms enough when $t = T^{\frac{1}{M}}$. One way to prove this is to use Theorem 2 in \cite{LR}. However, this has some disadvantages. First, it cannot be applied to the UCB adaptation. Second, it requires that we first prove the policy is consistent, which is very unnecessary. Finally, it cannot be translated into a finite-time result. Hence, we present a direct algorithm-oriented proof here. The key idea is, in order to prove something happens with a very low probability, we instead prove that its consequence happens with a very low probability.
\begin{proof}[Proof sketch of \lemmaref{lower1}]
    Let $\xi$ be a shorthand for $1 \di D_{\KL}\left(\B(\mu_a)||\B(\mu^*)\right)$. It suffices to prove that in a single-player setting, $\lim_{T\to\infty}\Pr\left(N_T(a) < (1 - \delta)\xi\ln(T)\right) = 0$\footnote{Since we are talking about a single-player setting, here and later, the subscript $p$ is dropped.}, then use a union bound to get the desired result. Define the random variable $\Upsilon_T$ to be the arm that is pulled most frequently before round $T$ and $\Psi$ to be the last round before round $T$ that $\Upsilon_T$ is pulled. Then
    \begin{align*}
        \Pr\left(N_T(a) < (1 - \delta)\xi\ln(T)\right) &\leq \Pr\left(N_T(a) < (1 - \delta)\xi\ln(T)\wedge B^+_{\Psi_T}(\Upsilon_T)\geq B^+_{\Psi_T}(a)\right)\\
                                                       &\leq \Pr\left(B_{\Psi_T}^+(\Upsilon_T) \geq \mu_{\Upsilon_T} + \epsilon\right) + \Pr\left(\tilde{\mu}_{\Psi_T}(a) \leq \mu_a - \epsilon\right) \numberthis\label{goesto0}\\
        \text{\big($\epsilon$ chosen to make this equals $0$\big)}&+ \Pr\Big(D_{\KL}\left(\B(\mu_a - \epsilon)||\B(\mu^* + \epsilon)\right) > \frac{D_{\KL}\left(\B(\mu_a)||\B(\mu^*)\right)}{(1 - \delta)}\Big).
    \end{align*}
    By definitions we have $N_{\Psi_T}(\Upsilon_T) \geq (T - 1) \di M$, therefore $B^+_{\Psi_T}(\Upsilon_T)\to\tilde{\mu}_{\Psi_T}(\Upsilon_T)$.
    Now we need to show $N_t(a)\to\infty$ in probability. Then use the fact that $\Psi_T \geq (T - 1) \di M$, as well as  Hoeffding's inequality, we can prove the two terms in \eqref{goesto0} converge to $0$ as $T$ goes to infinity.
\end{proof}
Then we can start our proof of \theoremref{onethm}. The key observation is to define random time points $\Phi_T, \Psi_T, \Lambda_{p, t}$, bound the count after $\Lambda_{p, t}$ using a standard bandit argument, and then bound the count before $\Lambda_{p, t}$ using \lemmaref{lower1}.
\begin{proof}[Proof sketch of \theoremref{onethm}]
    Let $\delta > 0$ be an arbitrarily small number, $\xi$ be a shorthand for $1 /D_{\KL}\left(\B(\mu_a)||\B(\mu^*)\right) $, and $T_0$ be a shorthand for $\lceil T^{\frac{1}{M}}\rceil$. We define random variables $\Phi_T$ and $\Psi_T$ in the following way: if $N_{T_0}(a) \geq (1 - \delta)\xi\ln(T)$, then $\Phi_T = 0$ and $\Psi_T = T_0$; otherwise $\Phi_T = T_0$ and $\Psi_T = T$. We also define random variables $\Lambda_{p, T} = \max\{t: \Phi_T \leq t< \Psi_T, N_{p, t}(a) < (1 - \delta)\xi\ln(\Psi_T)\}$. It can be checked that $\Lambda_{p, T}$ is well-defined. By a standard bandit argument, for each player $p$ the expected number of pulls of arm $a$ after $\Lambda_{p, t}$ is no more than $4\delta\xi\ln(T) + o(\ln(T))$, which is negligible. The rest expected number of pulls is no more than
    \begin{align*}
        M\cdot\E\left[N_{1, \Lambda_{1, T}}(a)\right] &\leq M\cdot\left(\Pr(\Phi_T = 0)(1 - \delta)\xi\ln(T_0) + \Pr(\Phi_T = T_0)(1 - \delta)\xi\ln(T)\right)\numberthis\label{similar}\\
        \why{by \lemmaref{lower1}}     &\leq \xi(1 - \delta)M\ln(\lceil T^{\frac{1}{M}}\rceil) + o(1) \cdot M(1 - \delta)\xi\ln(T)\\
                                       &= \xi(1 - \delta)\ln(T) + o(\ln(T)).&\,\,\,\,\,\,\,\Box
    \end{align*}\let\qed\relax
\end{proof}

\section{Stable Strategies}
\label{sustainable}

While the over-exploration strategy works like magic, it has two disadvantages. First, it requires a known time horizon, while in practice we often need an anytime algorithm. Second, when $T$ is relatively small, almost all the bandit policies do much better than the upper bound suggests, which makes the choice of $T^{\frac{1}{M}}$ smaller than actually needed. In other words, \lemmaref{lower1} would only make sense when $T^{\frac{1}{M}}$ is relatively large, which may require a huge $T$ if $M$ is large.

To develop anytime algorithms, we first introduce the concept of \textit{communication set}. The communication set is the set of all the communication rounds. Since we want to develop anytime algorithms, we assume the communication set $\C$ is an infinite set whose elements are denoted as $C_1, C_2, C_3, \dotsc$ in the increasing order. Then we need to measure the \textit{frequency} of the communication set.
\begin{definition}
    The counting function on a communication set $\C$ is defined by
    $    \ZC(n) \coloneqq \abs{\C\cup[n]}$.        
\end{definition}
\begin{definition}
    The density of a communication set $\C$ is defined by
    $ \alpha(\C) \coloneqq \liminf_{k\to\infty}\frac{\ln(C_k)}{\ln(C_{k + 1})}$.
    
\end{definition}
The counting function is a natural and intuitive way to define the frequency of communication. It basically tells us how many communication rounds there are in the first $n$ rounds for any $n$. However, as we will show later, the \textit{density} is the true property of a communication set in the bandit world. The relationship between these two measurements can be summarized by the following proposition.
\begin{proposition}\label{pro: three}
    For every communication set $\C$, 
    \textbf{(a)} if $ \ZC(n) \in o(\ln(\ln(n)))$, then  $\alpha(\C) = 0 $;
    \textbf{(b)}  if  $\alpha(\C) = 1$, then  $\ZC(n) \in \omega(\ln(\ln(n)))$;
    \textbf{(c)}      if $0 < \alpha(\C) < 1$, then \[\liminf_{n\to\infty}\frac{\ZC(n)}{\ln(\ln(n))\di\ln\left({\alpha(\C)}^{-1}\right)}\geq 1.\] 
\end{proposition}
In fact, \propositionref{pro: three}-(a) follows  directly from
\propositionref{pro: three}-(b) and (c). As we will show later, a higher density leads to a lower regret. Proposition~\ref{pro: three}-(b) says that in order to achieve the highest density, or equivalently the lowest regret, the number of communication rounds inevitably falls into the  class $\omega(\ln(\ln(n)))$. On the other hand, Proposition~\ref{pro: three}-(c) says that if we are aiming at a density greater than $0$ and less than $1$, or equivalently a regret that is ``not bad'',  then the number of communication rounds should be at least in the order of $\ln(\ln(n))\di\ln(\alpha(\C)^{-1})$.

Now we can use the concept of density to establish a lower bound.
\begin{theorem}
    \label{lowerbound}
    Let $D_{\inf}(\nu, a, \P) = \inf_{\nu'\in\P:\E[\nu'] > a}D_{\KL}(\nu||\nu')$, $\C$ be the communication set, and $\pi$ be a consistent bandit policy. Then for every suboptimal arm $a$ satisfying $0 < D_{\inf}(\nu_a, \mu^*, \P) < \infty$, 
    \[
        \limsup_{T\to\infty}\frac{\E[N^{\pi}_T(a)]}{\ln(T)} \geq \frac{M}{1 + (M - 1)\alpha(\C)}\cdot\frac{1}{D_{\inf}\left(\nu_a, \mu^*, \P\right)}.
    \]
\end{theorem}
To proof \theoremref{lowerbound}, we need to translate any distributed bandit into a single-player bandit, then perform a reduction from the single-player lower bound, and finally calculate out the new constant.
\begin{proof}[Proof sketch of \theoremref{lowerbound}]
    First, for each player $p$, we can use Theorem 1 in \cite{BK} to show that
    \[
        \liminf_{t\to\infty}\frac{\E[N_{p, t}(a)]}{\ln(t)}\geq\frac{1}{D_{\inf}(\nu_a, \mu^*, \P)}.\numberthis\label{notlemma}
    \]
    It can be done by relabel the round number in a distributed bandit and translate it into a serialized process, in which a player periodically changes his ``role'' so that he does not make decision based on all the data available in this serialized process. Then it remains to do some calculation. By the definition of $\alpha(\C)$, there exists a subsequence of $(C_k)_{k \geq 1}$, denoted by $(C_{n_s})_{s\geq 1}$, such that 
    \[\lim_{s\to\infty} \ln(C_{n_s}) \di \ln(C_{n_s + 1}) = \alpha(\C).\numberthis\label{sub}
    \]
    Let $x = \limsup_{T\to\infty}\E[N_T(a)] \di\ln(T)$. If $x = \infty$, the desired inequality holds trivially. Otherwise,
    \begin{align*}
    x&\geq\limsup_{s\to\infty}\frac{\E[N_{C_{n_s + 1}}(a)]}{\ln(C_{n_s + 1})} \geq \liminf_{s\to\infty}\frac{\E\left[\sum_{p = 1}^{M}N_{p, C_{n_s + 1} - 1}(a) - (M - 1)\cdot N_{C_{n_s}}(a)\right]}{\ln(C_{n_s + 1})}&\\
     &\geq \sum_{p = 1}^{M}\liminf_{s\to\infty}\frac{\E[N_{p, C_{n_s + 1} - 1}(a)]}{\ln(C_{n_s + 1})} - (M - 1)\cdot\limsup_{s\to\infty}\frac{\E[N_{C_{n_s}}(a)]}{\ln(C_{n_s + 1})}\\
     &\geq \sum_{p = 1}^{M}\liminf_{t\to\infty}\frac{\E[N_{p, t}(a)]}{\ln(t)} - (M - 1)\cdot\limsup_{s\to\infty}\frac{\E[N_{C_{n_s}}(a)]}{\ln(C_{n_s})}\cdot\lim_{s\to\infty}\frac{\ln(C_{n_s})}{\ln(C_{n_s + 1})}\\
        &\geq \frac{M}{D_{\inf}(\nu_a, \mu^*, \P)} - (M - 1)\cdot x\cdot\alpha(\C)\ \ \ \ \ \ \ \why{by \eqref{notlemma} and \eqref{sub}}.
    \end{align*}
    Solving $x$ concludes the proof.
\end{proof}

\subsection{Oblivious Policies Under Dense Communication Sets}
\label{dense}
\theoremref{lowerbound} shows that if we want to achieve the optimal regret, or in other words, if we do not want the regret scale with the number of players, then the density of the communication set must be $1$. Note that linear grid $\{d, 2d, 3d, \dotsc\}$ and exponential grid $\{q, q^2, q^3, \dotsc\}$ both have density $1$, while a double-exponential grid $\{q^{1 + \epsilon}, q^{(1 + \epsilon)^2}, q^{(1 + \epsilon)^3}, \dotsc\}$ has density $1 / (1 + \epsilon) < 1$. The following theorem shows the KL-UCB adaptation achieves this lower bound when $\alpha(\C) = 1$.
\begin{theorem}
    \label{main0}
    If the rewards are Bernoulli rewards and the communication set $\C$ satisfies $\alpha(\C) = 1$, then for every suboptimal arm $a$, 
    \begin{align*}
        \E\left[N_T^{\KLUCB}(a)\right] &\leq \frac{\ln(T)}{D_{\KL}\left(\B(\mu_a)||\B(\mu^*)\right)} + o\left(\ln(T)\right).
    \end{align*}
\end{theorem}
Due to the limited space, we defer our proof to \sectionref{sparse}, where the DKLUCB policy is introduced. DKLUCB, as a generalization of the KL-UCB adaptation, is optimal even if $\alpha(\C) < 1$.
\subsection{Non-Oblivious Policies for Sparse Communication Sets}
\label{sparse}

In \sectionref{dense} we showed that the KL-UCB adaptation is optimal for dense communication sets (i.e., $\alpha(\C) = 1$). However, if the communication set is very sparse (i.e., $\alpha(\C) < 1$), then we cannot expect an oblivious policy to do uniformly well. As Strategy C in \exampleref{example} has demonstrated, the main difficulty in designing an algorithm for the distributed MAB problem is that each player is ``isolated'' from others during the period between two communication rounds.
In the single-player setting, if one player pulled a suboptimal arm, he should be more certain that this arm is not optimal. Therefore, he should  explore this suboptimal arm less frequently in future rounds. However, when it comes to the distributed setting, although each player can utilize the information produced by a suboptimal decision made by himself immediately, other players would not know this experience until next communication round.

Based on the observations above, the key idea is to make each player explore less, since the results of exploration will become common knowledge at the next communication stage. This is implemented by having each player attempt to predict the number of pulls from each arm made by other players since the last communication round. These predictions can be wrong, but not that wrong as all processes are running the same algorithm. In fact, we can show that the errors in these count predictions are negligible as $T$ goes to infinity. However, even if these count predictions were fully correct, we still cannot simulate the full-communication KL-UCB adaptation. This is because the number of the data available is less than the count prediction. For this reason, we have to use a larger confidence bound. The following is a non-oblivious distributed policy called DKLUCB (stands for distributed KL-UCB), where we replace the count $N_{p, t}(a)$ in the KL-UCB adaptation with the count prediction $N'_{p, t}(a)$, and replace the original confidence bound with a slightly larger one.

\textbf{DKLUCB for Bernoulli rewards.} Define $\ell(t)\coloneqq \max\left\{u \leq t: u\in\C \vee u = 0\right\}$, 
\[
    \F(t) = \frac{M (\ln(t) + 3\ln\left(\ln(t)\right))}{1 + (M - 1)\alpha(\C)}, \,u_t(a) = \frac{N_{\ell(t)}(a)}{M}\left(\frac{1}{\alpha(\C)} - 1\right),
\]
\[
    N'_{p, t}(a) = N_{p, t}(a) + (M - 1)\cdot\min\left(N_{p, t}(a) - N_{\ell(t)}(a), u_t(a)\right).
\]
Then choose $A_{p, t} = \argmax_{a}B_{p, t}^+(a)$,
\[
    \text{where\ \ }B^+_{p, t}(a) = \sup\left\{p \in (0, 1): D_{\KL}\left(\hat{\nu}_{\X_{p, t - 1}(a)}||\B(p)\right) \leq \frac{\F\left(\sum_{k = 1}^K N_{p, t - 1}(k)\right) }{N'_{p, t - 1}(a)}\right\}.
\]
The DKLUCB policy is a non-oblivious policy because player $p$ in round $t$ makes decision not only based on $\X_{p, t-1}(a)$, but also based on $N_{\ell(t - 1)}(a)$, which is the number of pulls of each arm at the end of last communication round. Note that the KL-UCB adaptation can be seen as a special case of the DKLUCB policy. In fact, when $\alpha(\C) = 1$, DKLUCB is identical to the KL-UCB adaptation.
\begin{theorem}
    \label{main}
    If the rewards are Bernoulli rewards and the communication set is $\C$, then for every suboptimal arm $a$, 
    \begin{align*}
        \E\left[N_T^{\DKLUCB}(a)\right] &\leq \frac{M}{1 + (M - 1)\alpha(\C)}\cdot\frac{\ln(T)}{D_{\KL}\left(\B(\mu_a)||\B(\mu^*)\right)} + o\left(\ln(T)\right).
    \end{align*}
\end{theorem}
This theorem shows that the DKLUCB policy can achieve the lower bound in \theoremref{lowerbound}. This is consistent with the existing results and intuition. If $M = 1$, then DKLUCB is identical to the single-player KL-UCB policy, and the upper bound is the same. If $\alpha(\C) = 1$, DKLUCB is identical to the KL-UCB adaptation, and the upper bound is still the same (this is formalized as \theoremref{main0}). If $\alpha(\C) = 0$, DKLUCB is no better than $M$ single-player KL-UCB policies running independently.

As in \sectionref{one}, to prove an upper bound, first a lower bound is needed. However, this time we need a lemma much stronger than \lemmaref{lower1}, and theorem 2 in \cite{LR} definitely will not help.
\begin{lemma}
\label{lower2}
    If the rewards are Bernoulli rewards and the communication set is $\C$, then for every suboptimal arm $a$ and any $\delta > 0$, 
    \begin{align*}
        \lim_{T\to\infty}\Pr\left(\bigcap_{t\geq T}\left\{N^{\DKLUCB}_{t}(a) \geq \frac{M}{1 + (M - 1)\alpha(\C)}\cdot\frac{(1 - \delta)\ln(t)}{D_{\KL}(\B(\mu_a)||\B(\mu^*))}\right\}\right) = 1.
    \end{align*}
\end{lemma}
The idea of the proof is similar to that of \lemmaref{lower1}, but there are many new ingredients. First we require the following claim to relate the count prediction $N'_{p, t}(a)$ to the true count $N_t(a)$.
\begin{claim}\label{claimforgreater}
    For every $a$, $t$, and constant $c$, if $N_{t}(a) < c$, then there exists $p$ such that $N'_{p, t}(a) < c$.
\end{claim}
\begin{proof}[Proof sketch of \lemmaref{lower2}]
    Let $\xi$ be a shorthand for the leading constant before $\ln(t)$. It suffices to prove $\lim_{T\to\infty}\Pr(\bigcup_{t\geq T}\left\{N_{t}(a) \leq (1 - \delta)\xi\ln(t)\right\}) = 0. $ For every round $t$, we define a random player $p_t$ such that $N'_{p_t, t}(a) \leq N'_{p', t}(a)$ for every $p'\in[M]$. We say player $p_t$ is \textit{chosen} for round $t$. We also define random players $\Gamma_t$ such that $\Gamma_t$ is the player who is chosen most frequently before round $t$. Let $\Xi_t$ be a random set including the rounds (before round $t$) in which player $\Gamma_t$ is chosen. Let random variable $\Upsilon_t$ be the arm that is pulled most frequently by $\Gamma_t$ in those rounds in $\Xi_t$, and let $\Psi_t$ be the last round in $\Xi_t$ that $\Gamma_t$ pulls $\Upsilon_t$. By \claimref{claimforgreater},
    \begin{align*}
        \Big\{N_{t}(a) \leq (1 - \delta)\xi\ln(t)\Big\}&\subseteq\Big\{N'_{\Gamma_t, \Psi_t}(a) \leq (1 - \delta)\xi\ln(t)\wedge B^+_{\Gamma_t, \Psi_t}(\Upsilon_t)\geq B^+_{\Gamma_t, \Psi_t}(a)\Big\}.
    \end{align*}
    By definitions of $\Upsilon_T$, $\Gamma_T$, and $\Psi_T$, we have $\Psi_T \geq  (t - 1) \di (MK)$ and $N_{\Gamma_t, \Psi_t}(\Upsilon_t)\geq (t - 1) \di (MK)$. 
    The rest is close to the proof of \lemmaref{lower1}, requiring more careful dealing with the infinite union.
\end{proof}
The following claim explains why we need a slightly larger confidence bound in DKLUCB.
\begin{claim}\label{rel}
    For every player $p$, arm $a$, and $t \geq 0$, $N' _{p, t}(a) \leq M \di\left(1 + (M - 1)\alpha(\C)\right)\cdot N_{p, t}(a).$
\end{claim}
In the proof of \theoremref{onethm}, choosing $\Lambda_{p, t}$ is relatively easy. However, \theoremref{main} requires choosing $\Lambda_{p, T}$ more cleverly. And the overlapping histories cannot be simply ignored like in \theoremref{onethm}.
\begin{proof}[Proof sketch of \theoremref{main}]
    Let $\delta > 0$ be an arbitrarily small number and $\xi$ be a shorthand for the leading constant before $\ln(T)$. Define $\Upsilon_T$ to be the largest positive integer such that $C_{\Upsilon_T} \leq T$ and $N_{C_{\Upsilon_T}}(a) < (1 - \delta)\xi\ln(T)$. Define random variables $\Lambda_{p, T}$ to be the last round such that $C_{\Upsilon_T} \leq\Lambda_{p, T}< C_{\Upsilon_T + 1}$ and $ N' _{p, \Lambda_{p, T}}(a) < (1 - \delta)\xi\ln(\min(C_{\Upsilon_T + 1}, T))$, or $C_{\Upsilon_T}$ if there is no such round. Using \claimref{rel}, a standard bandit argument would show that for each player $p$ the expected number of pulls of arm $a$ after $\Lambda_{p, t}$ is no more than $4\delta\xi\ln(T) + o(\ln(T))$, which is negligible. Using a decomposition similar to \eqref{similar}, by \lemmaref{lower2}, the rest expected number of pulls is
    \begin{align*}
        &M\cdot\E[N_{1, \Lambda_{1, T}}(a)]  - (M - 1)\E[N_{\Upsilon_T}(a)] = \E[N_{1, \Lambda_{1, T}}(a) + (M - 1)(N_{1, \Lambda_{1, T}}(a) - N_{C_{\Upsilon_T}}(a))] \\
        &\leq\E[N_{1, \Lambda_{1, T}}(a) + (M - 1)(N_{1, \Lambda_{1, T}}(a) - N_{C_{\Upsilon_T}}(a)) \given[\big] N_{C_{\Upsilon_T}}(a) \geq (1 - \delta)\xi\ln(T)] + o(\ln(T)).
    \end{align*}
By the definition of $N'_{p, t}(a)$ and $\Lambda_{p, T}$, we have \[N' _{1, \Lambda_{1, T}}(a)  = N_{1, \Lambda_{1, T}}(a) + (M - 1)\cdot\min\left(N_{1, \Lambda_{1, T}}(a) - N_{C_{\Upsilon_T}}(a), u_{\Lambda_{1, T}}(a)\right) < (1 - \delta)\xi\ln(T),\] this plus the definition of $\alpha(\C)$ will together imply that, given $N_{C_{\Upsilon_T}}(a) \geq (1 - \delta)\xi\ln(T)$, \begin{align*}\hspace{5em}N_{1, \Lambda_{1, T}}(a) + (M - 1)(N_{1, \Lambda_{1, T}}(a) - N_{C_{\Upsilon_T}}(a))  < \xi\ln(T) + o(\ln(T)).\hspace{5em}\Box\end{align*} \let\qed\relax
\end{proof}

\section{Discussion}
\label{beyond}

\textbf{Finite-time analysis.} We write all the analyses in a way such that they can be translated into finite-time analyses and produce finite-time results (e.g., we avoid using Theorem 2 of \cite{LR} to get \lemmaref{lower1}). However, our results heavily rely on \lemmaref{lower1} and \lemmaref{lower2}, where the technique we used works fine to show the probability converges to $1$, yet will produce horrible constant if translated into a finite-time version. Specifically, every time we use pigeonhole principle (e.g., ``the arm that is pulled most frequently"), we will add a constant $K$ or $M$, which we guess is not necessary in the finite-time bound. Our hypothesis is that there are more elegant ways to prove \lemmaref{lower1} and \lemmaref{lower2}, which will provide tighter constants. We hope we can solve this problem in future works.

\textbf{Difference between bandits and distributed bandits.} The typical way to prove an upper bound for traditional bandits is to show that once the upper bound is reached, later suboptimal decisions are negligible. However, this does not work for distributed bandits. The main difficulty is that even if the upper bound is reached \textit{globally} ($N_t(a)$ has reached the upper bound), it may not be reached \textit{locally} ($N_{p, t}(a)$ may be way less than the upper bound). That is why we need the density of communication set, \lemmaref{lower1}, and \lemmaref{lower2} to make sure the desired upper bound is reached locally.

\textbf{Beyond UCB and KL-UCB.} Most of the results and analyses in this paper are very specific, they are either restricted to the KL-UCB adaptation (or the UCB adaptation, after some modification or even simplification, as we have mentioned), or to a generalization of the KL-UCB adaptation (i.e., DKLUCB). However, results similar to \theoremref{onethm}, \theoremref{main0}, and \theoremref{main} can be reproduced for any UCB-like oblivious bandit policies, as long as they behave normally in the sense that properties similar to \lemmaref{lower1} and \lemmaref{lower2} hold. These results can also be extended to other oblivious bandit policies that are not based on confidence bound such as Thompson sampling \cite{T}. The intuition behind this is, all the bandit policies behave similarly, and possibly indistinguishable by observing the actions they take. This inspires us to develop a universal framework to prove the performance of bandit policies under distributed settings. However, this framework requires much more insights and we would like to leave it as future work.
\bibliography{ref}

\begin{thebibliography}{15}
\providecommand{\natexlab}[1]{#1}
\providecommand{\url}[1]{\texttt{#1}}
\expandafter\ifx\csname urlstyle\endcsname\relax
  \providecommand{\doi}[1]{doi: #1}\else
  \providecommand{\doi}{doi: \begingroup \urlstyle{rm}\Url}\fi

\bibitem[Auer et~al.(2002)Auer, Cesa-Bianchi, and Fischer]{ACF}
P.~Auer, N.~Cesa-Bianchi, and P.~Fischer.
\newblock Finite time analysis of the multiarmed bandit problem.
\newblock \emph{Machine Learning}, 47\penalty0 (2-3):\penalty0 235--256, 2002.

\bibitem[Auer et~al.(2009)Auer, Jaksch, and Ortner]{AJO}
P.~Auer, T.~Jaksch, and R.~Ortner.
\newblock Near-optimal regret bounds for reinforcement learning.
\newblock In \emph{Advances in Neural Information Processing Systems (NIPS)},
  2009.

\bibitem[Awerbuch and Kleinberg(2008)]{AK}
B.~Awerbuch and R.~Kleinberg.
\newblock Competitive collaborative learning.
\newblock \emph{J. Comput. Syst. Sci.}, 74\penalty0 (8):\penalty0 1271--1288,
  2008.

\bibitem[Burnetas and Katehakis(1996)]{BK}
A.~Burnetas and M.~Katehakis.
\newblock Optimal adaptive policies for sequential allocation problems.
\newblock \emph{Advances in Applied Mathematics}, 17\penalty0 (2):\penalty0
  122--142, 1996.

\bibitem[Capp\'e et~al.(2013)Capp\'e, Garivier, Maillard, Munos, and
  Stoltz]{CGMMS}
O.~Capp\'e, A.~Garivier, O.-A. Maillard, R.~Munos, and G.~Stoltz.
\newblock {K}ullback-{L}eibler upper confidence bounds for optimal sequential
  allocation.
\newblock \emph{Annals of Statistics}, 41\penalty0 (3):\penalty0 1516--1541,
  2013.

\bibitem[Garivier and Capp\'e(2011)]{GC}
A.~Garivier and O.~Capp\'e.
\newblock The {KL-UCB} algorithm for bounded stochastic bandits and beyond.
\newblock In \emph{Proc. of Conference on Learning Theory (COLT)}, 2011.

\bibitem[Hillel et~al.(2013)Hillel, Karnin, Koren, Lempel, and Somekh]{HKKLS}
E.~Hillel, Z.~Karnin, T.~Koren, R.~Lempel, and O.~Somekh.
\newblock Distributed exploration in multi-armed bandits.
\newblock In \emph{Advances in Neural Information Processing Systems (NIPS)},
  2013.

\bibitem[Honda and Takemura(2010a)]{HT}
J.~Honda and A.~Takemura.
\newblock An asymptotically optimal bandit algorithm for bounded support
  models.
\newblock In \emph{Proc. of Conference on Learning Theory (COLT)}, 2010a.

\bibitem[Joulani et~al.(2013)Joulani, Gy{\"o}rgy, and Szepesv\'ari]{JGS}
P.~Joulani, A.~Gy{\"o}rgy, and C.~Szepesv\'ari.
\newblock Online learning under delayed feedback.
\newblock In \emph{Proc. of International Conference on Machine Learning
  (ICML)}, 2013.

\bibitem[Kaufmann et~al.(2012)Kaufmann, Korda, and Munos]{KKM}
E.~Kaufmann, N.~Korda, and R.~Munos.
\newblock Thompson sampling: An asymptotically optimal finite time analysis.
\newblock \emph{Algorithmic Learning Theory}, 7568:\penalty0 199--213, 2012.

\bibitem[Lai and Robbins(1985)]{LR}
T.~Lai and H.~Robbins.
\newblock Asymptotically efficient adaptive allocation rules.
\newblock \emph{Advances in Applied Mathematics}, 6\penalty0 (1):\penalty0
  4--22, 1985.

\bibitem[Liu and Zhao(2010)]{LZ}
K.~Liu and Q.~Zhao.
\newblock Distributed learning in multi-armed bandit with multiple players.
\newblock \emph{IEEE Transactions on Signal Procesing}, 58\penalty0
  (11):\penalty0 5667--5681, Nov. 2010.

\bibitem[Maillard et~al.(2011)Maillard, Munos, and Stoltz]{MMS}
O.-A. Maillard, R.~Munos, and G.~Stoltz.
\newblock A finite-time analysis of multi-armed bandits problems with
  kullback-leibler divergences.
\newblock In \emph{Proc. of Conference on Learning Theory (COLT)}, 2011.

\bibitem[Sz{\"{o}}r{\'{e}}nyi et~al.(2013)Sz{\"{o}}r{\'{e}}nyi, Busa{-}Fekete,
  Heged{\"{u}}s, Orm{\'{a}}ndi, Jelasity, and K{\'{e}}gl]{SBHOJK}
B.~Sz{\"{o}}r{\'{e}}nyi, R.~Busa{-}Fekete, I.~Heged{\"{u}}s, R.~Orm{\'{a}}ndi,
  M.~Jelasity, and B.~K{\'{e}}gl.
\newblock Gossip-based distributed stochastic bandit algorithms.
\newblock In \emph{Proc. of International Conference on Machine Learning
  (ICML)}, 2013.

\bibitem[Thompson(1933)]{T}
W.~R. Thompson.
\newblock On the likelihood that one unknown probability exceeds another in
  view of the evidence of two samples.
\newblock \emph{Biometrika}, 25\penalty0 (3-4):\penalty0 285--294, 1933.

\end{thebibliography}
\begin{appendices}
\section{Preliminaries}
In this section, we review some notations and introduce new ones that will be used in later proofs. We denote the set of all the positive integers by $\N^+$ and the set of all the positive real numbers by $\R^+$. We denote the set $\{1, 2, 3, \cdots, K\}$ by $[K]$. We denote by $\abs{A}$ the cardinality of a set $A$. The mean of distribution $\nu$ is denoted by $\E[\nu]$. We use $\wedge$ to represent logical conjunction (AND) and use $\vee$ to represent logical disjunction (OR). ``$\wedge$'' has higher precedence than ``$\vee$''. Both ``$\wedge$'' and ``$\vee$'' have higher precedence than other connectives such as ``$=$'' or ``$\prec$''.

\subsection{Distributed Bandit Process}
The communication set is an infinite set that contains the indices of the communication rounds and is always denoted by $\C$. The elements of $\C$ are denoted by $C_1, C_2, C_3, \dotsc$ in the ascending order. 
We define the function $\ell: \N^+\to\N$ by 
\[\ell(t) \coloneqq \max\{u \leq t: u \in \C \vee u = 0\}.
\] 
That is, $\ell(t)$ is the last communication round in the first $t$ rounds, and it takes value $0$ if there is no such a round. We now define a strict partial order $\prec$ on all the rewards $\mathcal{X} = \{X_{u, v}: u\geq 1, v\geq 1\}$. Concretely, 
$X_{u_1, v_1}\prec X_{u_2, v_2}$ if and only if 
\[ 
  v_1 \leq \ell(v_2) \vee u_1 = u_2 \wedge v_1 < v_2.
\]  
That is, $X_{u_1, v_1}\prec X_{u_2, v_2}$ if and only if anyone who has collected reward $X_{u_2, v_2}$ must also have collected reward $X_{u_1, v_1}$. For each player $p$, we define $\prec_p^*$ to be a linear extension of $\prec$ such that $X_{u_1, v_1} \prec_p^* X_{u_2, v_2}$ if and only if
\[
    v_1 \leq \ell(v_2) \vee \ell(v_1) = \ell(v_2) \wedge (u_1 = p \vee u_1 < u_2)
\]
Note the subscript $p$ means this linear extension is defined differently for each player. For player $p$, this linear extension gives an order on the rewards he receives. It can be checked that both $\prec$ and $\prec^*_p$ are legitimate definitions. We also define 
\[\mathcal{X}_{p, t}(a) \coloneqq \{X_{u, v}: X_{u, v} \prec X_{p, t + 1} \wedge A_{u, v} = a\} = \{X_{u, v}: X_{u, v} \prec_p^* X_{p, t + 1} \wedge A_{u, v} = a\}
\] 
and $N_{p, t}(a) \coloneqq \abs{\mathcal{X}_{p, t}(a)}$. That is, $\X_{p, t}(a)$ is the set of rewards from arm $a$ that player $p$ has collected at the end of round $t$, and $N_{p, t}(a)$ is the cardinality of this set.

For each player $p$ and each arm $a$, we define a sequence of random variables $(X^{p, a}_{i})_{i \geq 1}$ where $X^{p, a}_i$ is the $i_{th}$ element in the set \[\left\{X_{u, v}: X_{u, v} \in \mathcal{X} \wedge A_{u, v} = a\right\}\] with respect to the order $\prec_p^*$. We also define \[\hat{\mu}_{p, s}(a) = \left(\sum_{i = 1}^s X^{p, a}_i\right) \di s.\] By this definition we can see that $\tilde{\mu}_{p, t}(a) = \hat{\mu}_{p, N_{p, t}(a)}(a)$. 

Recall that the random variable $\tilde{\mu}_{p, t}(a)$ is the empirical mean of those rewards generated by arm $a$ that are \textit{known} to player $p$ at the end of round $t$. There are $N_{p, t}(a)$ of these rewards, each of them obtained either by player $p$ pulling arm $a$ himself, or via communication (i.e., from other players). Thus we need a lemma to ensure that the additional data obtained from other players are indistinguishable from the data collected by players themselves. In other words, we shall prove that $\hat{\mu}_{p, s}(a)$ is the empirical mean of $s$ mutually independent random variables with the same distribution $\nu_a$. 
\begin{lemma}\label{lemma: binom}
    For every player $p$, every arm $a$, and every $s\in \N^+$, \[s\cdot\hat{\mu}_{p, s}(a)\sim B(s, \mu_a),\] where $B(\cdot, \cdot)$ is a binomial distribution.
\end{lemma}
\begin{proof}
    Fix an arm $a$ and a player $p$. We denote by $F(x)$ the cumulative distribution function corresponding to $\nu_a$. In addition, we define random variables $p_i$ and $t_i$ such that $p_i$ is the player who first receives the reward $X^{p, a}_i$ and $t_i$ is the round this receiving takes place. Clearly we have $X^{p, a}_i = X_{p_i, t_i}$, and therefore \[s\cdot\hat{\mu}_{p, s}(a) = \sum_{i = 1}^{s}X_{p_i, t_i}.\] Hence, now it suffices to show that $X_{p_i, t_i}$ are mutually independent random variables with the same cumulative distribution function $F(x)$.

First we will prove that for every $i$, $X_{p_i, t_i}$ has the cumulative distribution function $F(x)$. That is, $\Pr(X_{p_i, t_i} \leq \lambda) = F(\lambda)$ for any $\lambda \in \R$. Note that
    \begin{align}
        \Pr(X_{p_i, t_i} \leq \lambda) &= \sum_{u, v}\Pr\left(X_{u, v} \leq \lambda \wedge p_i = u \wedge t_i = v\right)\nonumber\\
                                         &= \sum_{u, v: \Pr\left(p_i = u\wedge t_i = v\right) > 0}\Pr\left(X_{u, v} \leq \lambda \given p_i = u \wedge t_i = v\right)\Pr\left(p_i = u \wedge t_i = v\right)\nonumber\\
                                         &= \sum_{u, v: \Pr\left(p_i = u\wedge t_i = v\right) > 0}F(\lambda)\Pr\left(p_i = u \wedge t_i = v\right) =F(\lambda)\label{basecase}
    \end{align}
    where the second equality holds because given $p_i = u$ and $t_i = v$, the reward $X_{u, v}$ is generated by arm $a$ independently.

    Then we will prove mutual independence by induction. It suffices to show that for every $s$ real numbers $\lambda_1, \lambda_2, \cdots, \lambda_s$ and $s$ integers $l_1 < l_2 < \cdots < l_s$ we have \[\Pr\left(\bigwedge_{i = 1}^s (X_{p_{l_i}, t_{l_i}} \leq \lambda_i)\right) = \prod_{i = 1}^sF(\lambda_i).\] 
    The base case where $s = 1$ has been proved in \eqref{basecase}. Suppose we have proved that mutual independence holds for $s - 1$, and we also define the event $\mathcal{E}_{s, u, v}$ by 
    \[
        \mathcal{E}_{s, u, v} \coloneqq \left\{\bigwedge_{i = 1}^{s - 1} (X_{p_{l_i}, t_{l_i}} \leq \lambda_i) \wedge p_{l_s} = u \wedge t_{l_s} = v\right\}.
    \]
    Then we have
    \begin{align*}
        \Pr\left(\bigwedge_{i = 1}^s (X_{p_{l_i}, t_{l_i}} \leq \lambda_i)\right) &=\sum_{u, v}\Pr\left(X_{u, v} \leq \lambda_s \wedge \mathcal{E}_{s, u, v}\right)\\
                                                                                  &=\sum_{u, v: \Pr(\mathcal{E}_{s, u, v}) > 0}\Pr\left(X_{u, v} \leq \lambda_s \given \mathcal{E}_{s, u, v}\right)\Pr\left(\mathcal{E}_{s, u, v}\right)\\
                                                                                  &=\sum_{u, v: \Pr(\mathcal{E}_{s, u, v}) > 0}F(\lambda_s)\Pr\left(\mathcal{E}_{s, u, v}\right)\\
                                                                                    &=F(\lambda_s)\prod_{i = 1}^{s - 1}\Pr\left(\bigwedge_{i = 1}^{s - 1} (X_{p_{l_i}, t_{l_i}} \leq \lambda_i)\right)\\ 
                                                                                    &=F(\lambda_s)\prod_{i = 1}^{s - 1}F(\lambda_i) \text{\ \ \ \ \ \ \ \ \ \ (by induction hypothesis)}\\                                                     
                                                                                    &=\prod_{i = 1}^sF(\lambda_i),
    \end{align*}
    where the third equality holds because given the event $\mathcal{E}_{s, u, v}$, the reward $X_{u, v}$ is generated by arm $a$ independently.
\end{proof}
Technically, this lemma is required whenever the Hoeffding's inequality is used to bound $\hat{\mu}_{p, s}(a)$. For simplicity, later proofs may use this lemma \textbf{without explicitly pointing it out}.
\subsection{Kullback-Leibler Divergences}

The Kullback-Leibler divergence (KL-divergence) from probability distribution $\nu_1$ to probability distribution $\nu_2$ is defined by
\[D_{\KL}(\nu_1||\nu_2) \coloneqq -\E_{\nu_1}\left[\ln\left(\frac{d\nu_2}{d\nu_1}\right)\right]. 
 \]
 Accordingly, we define \[D_{\inf}(\nu, a, \mathcal{P}) \coloneqq \inf_{\nu'\in\mathcal{P}:\E[\nu'] > a}D_{\KL}(\nu||\nu').\]
 Clearly, for two Bernoulli distribution $\nu_1$ and $\nu_2$ satisfying $\E[\nu_1] < \E[\nu_2]$, we have \[D_{\KL}(\nu_1||\nu_2) = \K_{\inf}(\nu_1, \E[\nu_2], \mathcal{B}),\] where $\mathcal{B}$ is the set of all the Bernoulli distributions.
Note that the parameter of a Bernoulli distribution usually takes value in the open interval $(0, 1)$. However, the empirical mean of Bernoulli trials can take value in the closed interval $[0, 1]$. Hence we define the extended Bernoulli distribution with parameter $p\in[0, 1]$ to be a distribution having probability mass $p$ on $1$ and $1 - p$ on $0$. 
We let 
\[\K(p, q) \coloneqq p\ln\left(\frac{p}{q}\right) + (1 - p)\ln\left(\frac{1 - p}{1 - q}\right)\]
be the KL-divergence from an extended Bernoulli distribution with parameter $p$ to another with parameter $q$, with conventions $0\cdot\ln(0) = 0$ and $\ln(0\di 0) = 0$. We also define the left-side truncated KL-divergence $\K'(p, q)$ as $0$ if $p > q$, or $\K(p, q)$ otherwise.
\subsection{Tools From Single-Player Bandits}
The original proof of the KL-UCB algorithm mainly relies on the following self normalized deviation bound, which we cannot avoid either.
\begin{lemma}[\cite{CGMMS}] 
    \label{lem: cite}
    Let $\hat{\mu}_s$ be the empirical mean of $s$ mutually independent Bernoulli random variables with the same mean $p$, then
    \begin{align*}
        \Pr\left(\bigcup_{s = 1}^{t}\left\{\hat{\mu}_{s} < p \wedge s\cdot\K\left(\hat{\mu}_{s}, p\right) \geq \epsilon\right\}\right)\leq e\lceil\epsilon\ln(t)\rceil e^{-\epsilon}.
    \end{align*}
\end{lemma}

\section{Proof of \theoremref{onethm}}
    Let $\delta > 0$ be an arbitrarily small number, $\xi$ be a shorthand for $1 /D_{\KL}\left(\B(\mu_a)||\B(\mu^*)\right) $, and $T_0$ be a shorthand for $\lceil T^{\frac{1}{M}}\rceil$. We define random variables $\Phi_T$, $\Psi_T$ in the following way: if $N_{T_0}(a) \geq (1 - \delta)\xi\ln(T)$, then $\Phi_T = 0$ and $\Psi_T = T_0$; otherwise $\Phi_T = T_0$ and $\Psi_T = T$. We also define random variables $\Lambda_{p, T} = \max\{t: \Phi_T \leq t< \Psi_T, N_{p, t}(a) < (1 - \delta)\xi\ln(\Psi_T)\}$. It can be checked that $\Lambda_{p, T}$ is well-defined. 
    
    \textbf{Step 1: Bound the count after $\Lambda_{p, T}$ (traditional bandit argument).} 
    We first do an event decomposition:
    \begin{align*}
        \left\{A_{p, t} = a\right\}\subseteq\left\{B^+_{p, t}(a^*) < \mu^*\right\} \cup \left\{B^+_{p, t}(a) \geq \mu^* \wedge A_{p, t} = a\right\}, \text{\ \ for $t$ large enough.}
    \end{align*}
    Then we will show that the event $\left\{B^+_{p, t}(a^*) < \mu^*\right\}$ can be safely ignored.
\begin{lemma}
    For every player $p$ and every arm $a$, 
    \begin{align*}
        \sum_{t = 1}^T\Pr\left(B^+_{p, t}(a) < \mu_a\right) = o\left(\ln(T)\right).
    \end{align*}
\end{lemma}
\begin{proof}
    Note that
    \begin{align}
        \Pr\left(B^+_{p, t}(a) < \mu_a\right) &\leq\Pr\left({N}_{p, t}(a)\K'(\hat{\mu}_{p, {N}_{p, t}(a)}(a), \mu_a) \geq \ln(t) + 3\ln(\ln(t))\right)\nonumber\\
        &\leq\Pr\left(\bigcup_{s = 1}^{Mt}\left\{\hat{\mu}_{p, s}(a) < \mu_a \wedge s\K(\hat{\mu}_{p, s}(a), \mu_a) \geq \ln(t) + 3\ln(\ln(t))\right\}\right)\nonumber\\
        &\leq \frac{e\lceil\left(\ln(t) + 3\ln(\ln(t))\right)\ln(Mt)\rceil}{t\ln^3(t)}\label{sumover},
    \end{align}
    where the last inequality follows from \lemmaref{lem: cite}. Sum \eqref{sumover} from $1$ to $T$ yields $o(\ln(T))$.
\end{proof}
    Hence we can ignore the event $\left\{B^+_{p, t}(a^*) < \mu^*\right\}$ and only bound the probability of the event $\left\{B^+_{p, t}(a) \geq \mu^* \wedge A_{p, t} = a\right\}$. 
    \begin{align*}
        &\E\left[\sum_{t = \Lambda_T + 1}^{T}\I_{\left\{B^+_{p, t}(a) \geq \mu^* \wedge A_{p, t} = a\right\}}\right] \\
        &\leq \E\left[\sum_{t = \Lambda_T + 1}^{T}\I_{\left\{\K'\left(\tilde{\mu}_{p, t}(a), \mu^*\right) \leq \frac{\K(\mu_a, \mu^*)}{1 + \delta}\wedge A_{p, t} = a\right\}}\right] + \E\left[\sum_{t = \Lambda_T + 1}^{T}\I_{\left\{N_{p, t}(a) < (1 + \delta)\xi\ln(t)\wedge A_{p, t} = a\right\}}\right] + o(\ln(T))\\
        &\leq\E\left[\sum_{s = 0}^{\infty}\I_{\left\{\K'\left(\hat{\mu}_{p, s}(a), \mu^*\right) \leq \frac{\K(\mu_a, \mu^*)}{1 + \delta}\right\}}\right] + 4\delta\xi\ln(T) + o(\ln(T))\text{\ \ \ \ (by the definition of $\Lambda_{p, T}$)}\\
                                                                                                                    &\leq\sum_{s = 0}^{\infty}\Pr\left(\hat{\mu}_{p, s}(a) \geq \mu + \epsilon\right)+ 4\delta\xi\ln(T) + o(\ln(T))\text{\ \ \ \ (for some $\epsilon > 0$)}\\
                                                                                                                    &\leq\sum_{s = 0}^{\infty}e^{-2\epsilon^2 s}+ 4\delta\xi\ln(T) + o(\ln(T))\text{\ \ \ \ (by Hoeffding's inequality)}\\
                                                                                                                    &=4\delta\xi\ln(T) + o(\ln(T)),
    \end{align*}
    Therefore, for all the players, the count after $\Lambda_{p, T}$ is no more than $4M\delta\xi\ln(T) + o(\ln(T))$.

    \textbf{Step 2: Bound the count before $\Lambda_{p, T}$.} The total count of all players before $\Lambda_{p, T}$ should be no more than
    \begin{align*}
        &\sum_{p = 1}^{M} \E\left[N_{p, \Lambda_{p, T}}(a)\right]  - (M - 1)\E\left[N_{\Phi_T}(a)\right]\\
        &\leq M\cdot\E\left[N_{1, \Lambda_{1, T}}(a)\right]\\
        &\leq M\cdot\left(\Pr(\Phi_T = 0)(1 - \delta)\xi\ln(T_0) + \Pr(\Phi_T = T_0)(1 - \delta)\xi\ln(T)\right)\\
        \why{by \lemmaref{lower1}}     &\leq \xi(1 - \delta)M\ln(\lceil T^{\frac{1}{M}}\rceil) + o(1) \cdot M(1 - \delta)\xi\ln(T)\\
                                       &= \xi(1 - \delta)\ln(T) + o(\ln(T)).
    \end{align*}
    \textbf{Step 3: Put everything together. }Adding up all the components, we get
    \[
        \E\left[N_T(a) \right] \leq \xi\left(1 + (4M - 1)\delta\right)\ln(T) + o(\ln(T)).
    \]
    This concludes the proof.
    
    \section{Proof of \lemmaref{lower1}}
    Let $\xi$ be a shorthand for $1 \di D_{\KL}\left(\B(\mu_a)||\B(\mu^*)\right)$. Then it suffices to prove that in a single player setting, $\lim_{T\to\infty}\Pr\left(N_T(a) \leq (1 - \delta)\xi\ln(T)\right) = 1.$ and then use a union bound.

    We can see this as weaker form of a special case of \lemmaref{lower2}. It is weaker because it does not contain a infinite intersection like \lemmaref{lower2}. It is a special case becuase we can let $M = 1$ in \lemmaref{lower2}. For these reasons, the proof should be a simplified version of the proof of \lemmaref{lower2}. To avoid duplication, we refer the reader to the proof in \sectionref{dup}.
    
\section{Proof of \propositionref{pro: three}}
Since \propositionref{pro: three}-(a) is a direct consequence of \propositionref{pro: three}-(b) and (c), it suffices to prove the latter two statements. Let \[\tilde{\C} = \{\ln(C_1), \ln(C_2), \ln(C_3), \cdots\}\] and we denote $\ln(C_k)$ by $\tilde{C}_k$.  
Assume $\alpha(\C) = d \in (0, 1]$. Then for every $\epsilon \in (0, d)$, there exists a $N_\epsilon$ large enough such that $\tilde{C}_k\di\tilde{C}_{k+1} \geq d - \epsilon$ for every $k \geq N_\epsilon$. Thus, for $n$ large enough,     
    \begin{align*}
        \ln(n) < \tilde{C}_{\ZC(n) + 1} \leq \frac{\tilde{C}_{N_\epsilon}}{(d - \epsilon)^{\ZC(n) + 1 - N_\epsilon}}.
    \end{align*}
    It then follows from the inequality that
    \begin{align*}
        \ZC(n) > \frac{\ln(\ln(n)) - \ln(\tilde{C}_{N_\epsilon})}{\ln(\frac{1}{d - \epsilon})}+ N_\epsilon - 1.
    \end{align*}
    Note that $\epsilon$ can be chosen to be arbitrarily small. If $d = 1$, then $\ZC(n)\in \omega(\ln\ln(n))$. Otherwise, we have \[\liminf_{n\to\infty}\frac{\ZC(n)}{\ln(\ln(n))\di\ln\left(d^{-1}\right)}\geq 1.\]
    This completes the proof of \propositionref{pro: three}-(b) and (c).

\section{Proof of \theoremref{lowerbound}}

For the single-player MAB model, \cite{LR} gave a lower bound for single-parametric distributions. \cite{BK} generalized this result to non-parametric models. Translated into our model, their results can be summarized as the following theorem. 

\begin{theorem}[\cite{BK}]\label{thm:BK} 
    If $M = 1$ and $\pi$ is a consistent policy, then for every suboptimal arm $a$ satisfying $0 < D_{\inf}(\nu_a, \mu^*, \P) < \infty$, 
    \[
        \liminf_{T\to\infty}\frac{\E[N^{\pi}_T(a)]}{\ln(T)} \geq \frac{1}{D_{\inf}\left(\nu_a, \mu^*, \P\right)}.
    \]
\end{theorem}
Now our goal is to establish a similar lower bound for the distributed case (i.e., $M > 1$). Fortunately, we can actually use \theoremref{thm:BK} as a stepping stone. More specifically,  we could use a \textit{simulator} to simulate all the actions of the $M$ players and apply \theoremref{thm:BK} to this single simulator. In other words, we treat every multi-player MAB process as a single-player MAB process whose outcome is indistinguishable from the original one. 
This conversion does not provide a direct solution to the lower bound for the distributed MAB problem. 
However, we can use it to establish a lower bound on $N_{p, t}(a)$ for every suboptimal arm $a$.
 
\begin{lemma}\label{lem:lower}
    If $\pi$ is a consistent policy, then for every suboptimal arm $a$ satisfying $0 < D_{\inf}(\nu_a, \mu^*, \P) < \infty$, 
    \[
        \liminf_{t\to\infty}\frac{\E[N^{\pi}_{p, t}(a)]}{\ln(t)} \geq \frac{1}{D_{\inf}\left(\nu_a, \mu^*, \P\right)}.
    \]
\end{lemma}
\begin{proof}
    Fix a player $p$ and a suboptimal arm $a$. Let $\mathcal{D}$ be the original distributed process. The rewards in the set $\mathcal{X}$ can be viewed as ones generated by a single-player MAB process $\mathcal{S}$ in order $\prec_p^*$ using a single-player policy $\pi'$. Now we use coupling to associate process $\mathcal{S}$ to process $\mathcal{D}$ and use superscripts to distinguish random variables in the two processes. Note that
    \begin{align}
        N_{t}^{\mathcal{S}}(a) \leq N_{p, t}^{\mathcal{D}}(a) \leq N_t^{\mathcal{D}}(a). \label{comp}
    \end{align}
    Note that the subscript $t$ in $N_{t}^{\mathcal{S}}(a)$ is referring to the time point $t$ in the serialized process, while the subscript $t$ in $N_{p, t}^{\mathcal{D}}(a)$ and $N_t^{\mathcal{D}}(a)$ are referring to the time point $t$ in the distributed process.

    Hence, if $\pi$ is a consistent policy (for the distributed MAB), then $\pi'$ is a consistent policy (for the single-player MAB). Thus,
    \[
        \liminf_{t\to\infty}\frac{\E[N^{\mathcal{D}}_{p, t}(a)]}{\ln(t)} \geq \liminf_{t\to\infty}\frac{\E[N^{\mathcal{S}}_{t}(a)]}{\ln(t)}\geq\frac{1}{D_{\inf}(\nu_a, \mu^*, \mathcal{P})}.
    \]
    where the first inequality follows from \eqref{comp} and the second follows from \theoremref{thm:BK}.
\end{proof}
Having a lower bound on $N_{p, t}(a)$ is almost equivalent to having a lower bound on $N_t(a)$. In fact, we have for every $t\geq 1$,
\begin{align}
    N_{t}(a) &\geq N_{\ell(t)}(a) + \sum_{p = 1}^M\left(N_{p, t}(a) - N_{\ell(t)}(a)\right) = \sum_{p = 1}^MN_{p, t}(a) - (M - 1)\cdot N_{\ell(t)}(a). \label{relation}
\end{align}
Using \eqref{relation} we can finish the proof of \theoremref{lowerbound}.
    Since $\alpha(\C) = \liminf_{k\to\infty}\frac{\ln(C_k)}  {\ln(C_{k + 1})}$, there exists a subsequence of $(C_k)_{k \geq 1}$, denoted by $(C_{n_s})_{s\geq 1}$, such that 
    \[\lim_{s\to\infty} \frac{\ln(C_{n_s})} {\ln(C_{n_s + 1})} = \alpha(\C).
    \]
    Let $x = \limsup_{T\to\infty}\frac{\E[N_T(a)]}{\ln(T)} \geq 0$. If $x = \infty$, the desired inequality holds trivially. 
    Assume $x$ is finite. Then
    \begin{align*}
    x&\geq\limsup_{s\to\infty}\frac{\E[N_{C_{n_s + 1}}(a)]}{\ln(C_{n_s + 1})} \geq \liminf_{s\to\infty}\frac{\E\left[\sum_{p = 1}^{M}N_{p, C_{n_s + 1} - 1}(a) - (M - 1)\cdot N_{C_{n_s}}(a)\right]}{\ln(C_{n_s + 1})}&\\
     &\geq \sum_{p = 1}^{M}\liminf_{s\to\infty}\frac{\E[N_{p, C_{n_s + 1} - 1}(a)]}{\ln(C_{n_s + 1})} - (M - 1)\cdot\limsup_{s\to\infty}\frac{\E[N_{C_{n_s}}(a)]}{\ln(C_{n_s + 1})}\\
     &\geq \sum_{p = 1}^{M}\liminf_{t\to\infty}\frac{\E[N_{p, t - 1}(a)]}{\ln(t)} - (M - 1)\cdot\limsup_{s\to\infty}\frac{\E[N_{C_{n_s}}(a)]}{\ln(C_{n_s})}\cdot\lim_{s\to\infty}\frac{\ln(C_{n_s})}{\ln(C_{n_s + 1})}\\
         &\geq \frac{M}{D_{\inf}(\nu_a, \mu^*, \mathcal{P})} - (M - 1)\cdot x\cdot\alpha(\C).\text{\ \ \ \ \ \ \ \ \ \ \ \ \ \ \ \ \ (by \lemmaref{lem:lower})}
    \end{align*}
    Solving $x$ concludes the proof.

\section{Proof of \theoremref{main0} and \theoremref{main}}
\theoremref{main0} is a special case of \theoremref{main}: If $\alpha(\C) = 1$, then the DKLUCB algorithm reduces to a simple adaptation of the KL-UCB, and the upper bound is identical to the single-player upper bound. Therefore, it suffices to prove \theoremref{main}.

    Let $\delta > 0$ be an arbitrarily small number and $\xi$ be a shorthand for the leading constant before $\ln(T)$. Define $\Upsilon_T$ to be the largest positive integer such that $C_{\Upsilon_T} \leq T$ and $N_{C_{\Upsilon_T}}(a) < (1 - \delta)\xi\ln(T)$. Define random variables $\Lambda_{p, T}$ to be the last round such that $C_{\Upsilon_T} \leq\Lambda_{p, T}< C_{\Upsilon_T + 1}$ and $ N' _{p, \Lambda_{p, T}}(a) < (1 - \delta)\xi\ln(\min(C_{\Upsilon_T + 1}, T))$. If there is no such round, let $\Lambda_{p, T} = C_{\Upsilon_T}$. 
    
    \textbf{Step 1: bound the count after $\Lambda_{p, T}$ (traditional bandit argument).}

    We first do an event decomposition:
    \begin{align*}
        \left\{A_{p, t} = a\right\}\subseteq\left\{B^+_{p, t}(a^*) < \mu^*\right\} \cup \left\{B^+_{p, t}(a) \geq \mu^* \wedge A_{p, t} = a\right\}, \text{\ \ for $t$ large enough.}
    \end{align*}
    Then we will show that the event $\left\{B^+_{p, t}(a^*) < \mu^*\right\}$ can be safely ignored.
\begin{lemma}
    For every player $p$ and every arm $a$, 
    \begin{align*}
        \sum_{t = 1}^T\Pr\left(B^+_{p, t}(a) < \mu_a\right) = o\left(\ln(T)\right).
    \end{align*}
\end{lemma}
\begin{proof}
    Note that
    \begin{align}
        \Pr\left(B^+_{p, t}(a) < \mu_a\right) &\leq\Pr\left(N'_{p, t}(a)\K'(\hat{\mu}_{p, {N}_{p, t}(a)}(a), \mu_a) \geq \xi\left(\ln(t) + 3\ln(\ln(t))\right)\right)\nonumber\\
        \why{by \claimref{rel}}   &\leq\Pr\left(N_{p, t}(a)\K'(\hat{\mu}_{p, {N}_{p, t}(a)}(a), \mu_a) \geq \ln(t) + 3\ln(\ln(t))\right)\nonumber\\
        &\leq\Pr\left(\bigcup_{s = 1}^{Mt}\left\{\hat{\mu}_{p, s}(a) < \mu_a \wedge s\K(\hat{\mu}_{p, s}(a), \mu_a) \geq \ln(t) + 3\ln(\ln(t))\right\}\right)\nonumber\\
        &\leq \frac{e\lceil\left(\ln(t) + 3\ln(\ln(t))\right)\ln(Mt)\rceil}{t\ln^3(t)}\label{sumover2},
    \end{align}
    where the last inequality follows from \lemmaref{lem: cite}. Sum \eqref{sumover2} from $1$ to $T$ yields $o(\ln(T))$.
\end{proof}
    Hence we can ignore the event $\left\{B^+_{p, t}(a^*) < \mu^*\right\}$ and only bound the probability of the event $\left\{B^+_{p, t}(a) \geq \mu^* \wedge A_{p, t} = a\right\}$. 
    \begin{align*}
        &\E\left[\sum_{t = \Lambda_T + 1}^{T}\I_{\left\{B^+_{p, t}(a) \geq \mu^* \wedge A_{p, t} = a\right\}}\right] \\
        &\leq \E\left[\sum_{t = \Lambda_T + 1}^{T}\I_{\left\{\K'\left(\tilde{\mu}_{p, t}(a), \mu^*\right) \leq \frac{\K(\mu_a, \mu^*)}{1 + \delta}\wedge A_{p, t} = a\right\}}\right] + \E\left[\sum_{t = \Lambda_T + 1}^{T}\I_{\left\{N'_{p, t}(a) < (1 + \delta)\xi\ln(t)\wedge A_{p, t} = a\right\}}\right] + o(\ln(T))\\
        &\leq\E\left[\sum_{s = 0}^{\infty}\I_{\left\{\K'\left(\hat{\mu}_{p, s}(a), \mu^*\right) \leq \frac{\K(\mu_a, \mu^*)}{1 + \delta}\right\}}\right] + 4\delta\xi\ln(T) + o(\ln(T))\text{\ \ \ \ (by the definition of $\Lambda_{p, T}$)}\\
                                                                                                                    &\leq\sum_{s = 0}^{\infty}\Pr\left(\hat{\mu}_{p, s}(a) \geq \mu + \epsilon\right)+ 4\delta\xi\ln(T) + o(\ln(T))\text{\ \ \ \ (for some $\epsilon > 0$)}\\
                                                                                                                    &\leq\sum_{s = 0}^{\infty}e^{-2\epsilon^2 s}+ 4\delta\xi\ln(T) + o(\ln(T))\text{\ \ \ \ (by Hoeffding's inequality)}\\
                                                                                                                    &=4\delta\xi\ln(T) + o(\ln(T)),
    \end{align*}
    Therefore, for all the players, the count after $\Lambda_{p, T}$ is no more than $4M\delta\xi\ln(T) + o(\ln(T))$.

    \textbf{Step 2: bound the count before $\Lambda_{p, T}$.}
    The total count of all players before $\Lambda_{p, T}$ should be no more than
    \begin{align*}
        &\sum_{p = 1}^M\E\left[N_{p, \Lambda_{p, T}}(a)\right] - (M - 1)\E\left[N_{C_{\Upsilon_T}}(a)\right]\\
        &=M\cdot\E\left[N_{1, \Lambda_{1, T}}(a)\right] - (M - 1)\E\left[N_{C_{\Upsilon_T}}(a)\right]\\
        &=\E\left[N_{1, \Lambda_{1, T}}(a) + (M - 1)(N_{1, \Lambda_{1, T}}(a) - N_{C_{\Upsilon_T}}(a))\right]\\
        &\leq\Pr\left(N_{C_{\Upsilon_T}} \geq (1 - \delta)\xi\ln(C_{\Upsilon_T})\right)\cdot\\
        &\hspace{3em}\E\left[N_{1, \Lambda_{1, T}}(a) + (M - 1)(N_{1, \Lambda_{1, T}}(a) - N_{C_{\Upsilon_T}}(a))\given N_{C_{\Upsilon_T}} \geq (1 - \delta)\xi\ln(C_{\Upsilon_T}) \right]\\
        &\hspace{1em}+\Pr\left(N_{C_{\Upsilon_T}} < (1 - \delta)\xi\ln(C_{\Upsilon_T})\right)\cdot\\
        &\hspace{3em}\E\left[N_{1, \Lambda_{1, T}}(a) + (M - 1)(N_{1, \Lambda_{1, T}}(a) - N_{C_{\Upsilon_T}}(a))\given N_{C_{\Upsilon_T}} < (1 - \delta)\xi\ln(C_{\Upsilon_T}) \right]\\
        &\leq\E\left[N_{1, \Lambda_{1, T}}(a) + (M - 1)(N_{1, \Lambda_{1, T}}(a) - N_{C_{\Upsilon_T}}(a))\given N_{C_{\Upsilon_T}} \geq (1 - \delta)\xi\ln(C_{\Upsilon_T}) \right]\numberthis\label{enough!}\\
        &\hspace{3em}+o(\ln(T)) \text{\ \ \ \ \ \ \ \ (by \lemmaref{lower2})}
    \end{align*}
    Our goal is to prove \eqref{enough!} is no more than $\xi\ln(T) + o(\ln(T))$. It can be done by showing $N_{1, \Lambda_{1, T}}(a) + (M - 1)(N_{1, \Lambda_{1, T}}(a) - N_{C_{\Upsilon_T}}(a)) \leq \xi\ln(T) + o(\ln(T))$ given $N_{C_{\Upsilon_T}} \geq (1 - \delta)\xi\ln(C_{\Upsilon_T})$. Now we suppose $N_{C_{\Upsilon_T}} \geq (1 - \delta)\xi\ln(C_{\Upsilon_T})$ holds. If $N_{1, \Lambda_{1, T}}(a) - N_{C_{\Upsilon_T}} \leq u_{\Lambda_{1, T}}(a)$, then it is trivial since in that case $N_{1, \Lambda_{1, T}}(a)+ (M - 1)(N_{1, \Lambda_{1, T}}(a) - N_{C_{\Upsilon_T}}(a)) = N'_{1, \Lambda_{1, T}}(a)$ and we already know $N'_{1, \Lambda_{1, T}}(a) < (1 - \delta)\xi\ln(T)$. Otherwise we have
    \begin{align*}
        N_{1, \Lambda_{1, T}}(a) + (M - 1)\frac{N_{C_{\Upsilon_T}}(a)}{M}\left(\frac{1}{\alpha(\C)} - 1\right) < (1 - \delta)\xi\ln\left(\min(T, C_{\Upsilon_T + 1})\right).
    \end{align*}
    Using the condition $N_{C_{\Upsilon_T}} \geq (1 - \delta)\xi\ln(C_{\Upsilon_T})$ we get
    \begin{align*}
        N_{1, \Lambda_{1, T}}(a) < (1 - \delta)\xi\left(\ln\left(\min(T, C_{\Upsilon_T + 1})\right) - \frac{M - 1}{M}\left(\frac{\ln(C_{\Upsilon_T})}{\alpha(\C)} - \ln(C_{\Upsilon_T})\right) \right).
    \end{align*}
    By the definition of $\alpha(\C)$ we have
    \begin{align*}
        N_{1, \Lambda_{1, T}}(a) &< (1 - \delta)\xi\left((1 + \delta)\frac{\ln\left(\min(T, C_{\Upsilon_T + 1})\right)}{M} + \frac{M - 1}{M}\ln(C_{\Upsilon_T})\right) + o(\ln(T))\\
                                 &<\frac{\xi\ln(T)}{M} + \frac{(1 - \delta)\xi(M - 1)}{M}\ln(C_{\Upsilon_T}) + o(\ln(T))\numberthis\label{bybyby}.
    \end{align*}
    Therefore, given $N_{C_{\Upsilon_T}} \geq (1 - \delta)\xi\ln(C_{\Upsilon_T})$, 
    \begin{align*}
        &N_{1, \Lambda_{1, T}}(a) + (M - 1)(N_{1, \Lambda_{1, T}}(a) - N_{C_{\Upsilon_T}}(a))\\
        &=M\cdot N_{1, \Lambda_{1, T}}(a) - (M - 1)\cdot N_{C_{\Upsilon_T}}(a)\\
        &\leq M\cdot N_{1, \Lambda_{1, T}}(a) - (M - 1)(1 - \delta)\xi\ln(C_{\Upsilon_T}) \\
        &\leq \xi\ln(T) + o(\ln(T)) \text{\ \ \ \ \ \ \ \ \ (by \eqref{bybyby})}
    \end{align*}
    Hence, 
    \begin{align*}
        \E\left[N_{1, \Lambda_{1, T}}(a) + (M - 1)(N_{1, \Lambda_{1, T}}(a) - N_{C_{\Upsilon_T}}(a))\given N_{C_{\Upsilon_T}} \geq (1 - \delta)\xi\ln(T) \right] \leq \xi\ln(T) + o(\ln(T))
    \end{align*}
    \textbf{Step 3: put everything together.} Adding all components up, we get \[\E[N_T(a)] \leq \xi(1 + 4M\delta)\ln(T) + o(\ln(T)).\] This concludes the proof.

\section{Proof of \lemmaref{lower2}}
\label{dup}
In this section, in order to make the proof easier and more readable, we will prove a lemma equivalent to \lemmaref{lower2}, We first introduce some new concepts.
\begin{definition}
    We say a sequence of random variables $\left(X_n\right)_{n\geq 1}$ 
    \begin{enumerate}[(a)] 
        \item
            converges to a constant $c$ in probability, denoted by $X_n\cinp c$, if for every $\epsilon > 0$, \[\lim_{n\to\infty}\Pr\left(\abs{X_n - c} \geq \epsilon\right) = 0;\]
        \item
            tends to infinity in probability, denoted by $X_n\cinp\infty$, if for every $N$, \[\lim_{n\to\infty}\Pr\left(X_n < N\right) = 0.\]
    \end{enumerate}
\end{definition}
The following is the equivalent lemma we will prove.
\begin{lemma}[equivalent to \lemmaref{lower2}]\label{lem: greater_re}
    \label{lem: greater}
Let $(\Phi_n)_{n\geq 1}$ be a sequence of random variables such that $\Phi_n\cinp\infty$. Then for every suboptimal arm $a$ and every $\delta > 0$, 
    \[
        \lim_{n\to\infty}\Pr\left(N_{{\Phi_n}}(a) \geq \frac{(1 - \delta)\mathcal{F}({\Phi_n})}{{\K}(\nu_a, \nu_{a^*})}\right) = 1.
    \]
\end{lemma}
Note that this is a classical technique to deal with infinite union (or intersection), we omit the proof of equivalence here.

To simplify our proof, we first present three utility lemmas. 
\begin{lemma}\label{lem:A}
    Let $(\Phi_n)_{n \geq 1}$ be a sequence of random variables such that $\Phi_n\cinp\infty$. Then for every player $p$ and every arm $a$, \[\hat{\mu}_{p, \Phi_n}(a)\cinp\mu_a.\]
\end{lemma}
\begin{proof}
    By the definition of convergence in probability, it suffices to show that for every $\delta > 0$ and every $\epsilon > 0$ we can find an $N_*$ such that for any $n \geq N_*$, \[\Pr\left(\abs{\hat{\mu}_{p, \Phi_n}(a) - \mu_a} \geq \delta \right)\leq \epsilon.\] Fix a $\delta > 0$ and $\epsilon > 0$, we can choose $N_0$ large enough such that
    \begin{align*}
        \Pr\left(\bigcup_{s = N_0}^{\infty}\abs{\hat{\mu}_{p, s}(a) - \mu_a} \geq \delta\right)
        \leq \sum_{s = N_0}^{\infty} \Pr\left(\abs{\hat{\mu}_{p, s}(a) - \mu_a} \geq \delta\right)
        \leq \sum_{s = N_0}^{\infty}2e^{-2\delta^2 s}
        \leq \frac{\epsilon}{2}.
    \end{align*}
    where the first inequality follows from union bound and the second follows from Hoeffding's inequality. Then by the definition of tending to infinity in probability, we can choose $N_1$ large enough such that \[\Pr\left(\Phi_n < N_0 \right) \leq \frac{\epsilon}{2}, \text{\ \ for every $n \geq N_1$}.\] Thus, for every $n \geq N_* = \max(N_0, N_1)$, 
    \begin{align*}
        \Pr\left(\abs{\hat{\mu}_{p, \Phi_n}(a) - \mu_a} \geq \delta \right) \leq \Pr\left(\Phi_n < N_0\right) + \Pr\left(\bigcup_{s = N_0}^{\infty}\abs{\hat{\mu}_{p, s}(a) - \mu_a} \geq \delta\right) \leq \epsilon,
    \end{align*}
    which concludes the proof.
\end{proof}
\begin{lemma}\label{lem:B}
    Let $(\Upsilon_n)_{n\geq 1}$ be a sequence of random arms, $(\Gamma_n)_{n\geq 1}$ be a sequence of random players, and $(\Phi_n)_{n\geq 1}$ be a sequence of random variables such that $\Phi_n\cinp\infty$. Then 
\begin{align}
    \exists \xi > 0, N_0 > 0, \forall n \geq N_0, N_{\Gamma_n, \Phi_n}(\Upsilon_n) \geq (\Phi_n)^{\xi}\implies \left(B_{\Gamma_n, \Phi_n}^+(\Upsilon_n) - \mu_{\Upsilon_n}\right)\cinp 0.\label{hhh}
\end{align}
\end{lemma}
\begin{proof}
    Assume the left hand side of \eqref{hhh} holds. By the definition of convergence in probability, it suffices to show that for every $\delta > 0$ we have 
    \[ \lim_{n\to\infty}\Pr\left(\abs{B_{\Gamma_n, \Phi_n}^+(\Upsilon_n) - \mu_{\Upsilon_n}} \geq \delta\right) = 0.\]
    Fix a $0 < \delta < 1 - \max_{a\in\A}\mu_a$, note that
    \begin{align}
        &\Pr\left(\abs{B_{\Gamma_n, \Phi_n}^+(\Upsilon_n) - \mu_{\Upsilon_n}} > \delta\right) \leq\Pr\left(\hat{\mu}_{\Gamma_n, N_{\Gamma_n, \Phi_n}(\Upsilon_n)}(\Upsilon_n) < \mu_{\Upsilon_n} - \delta\right)\nonumber \\
        &\hspace{2em}+ \Pr\left(\hat{\mu}_{\Gamma_n, N_{\Gamma_n, \Phi_n}(\Upsilon_n)}(\Upsilon_n) > \mu_{\Upsilon_n} + \frac{\delta}{2}\right) + \Pr\left(\K(\mu_{\Upsilon_n} + \frac{\delta}{2}, \mu_{\Upsilon_n} + \delta)\leq\frac{\mathcal{F}(\Phi_n)}{ N' _{\Gamma_n, \Phi_n}(\Upsilon_n)}\right)\nonumber\\
        &\hspace{10em}\leq \sum_{a\in\A}\sum_{p = 1}^M\Pr\left(\hat{\mu}_{p, N_{\Gamma_n, \Phi_n}(\Upsilon_n)}(a) < \mu_{a} - \delta\right) \label{111}\\
        &\hspace{10em}+ \sum_{a\in\A}\sum_{p = 1}^M\Pr\left(\hat{\mu}_{p, N_{\Gamma_n, \Phi_n}(\Upsilon_n)}(a) > \mu_{a} + \frac{\delta}{2}\right)\label{222}\\
        &\hspace{10em}+ \Pr\left(\min_{a\in\A}\K(\mu_{a} + \frac{\delta}{2}, \mu_{a} + \delta)\leq\frac{\mathcal{F}(\Phi_n)}{ N' _{\Gamma_n, \Phi_n}(\Upsilon_n)}\right).\label{333}
    \end{align}
    Recall that $N_{\Gamma_n, \Phi_n}(\Upsilon_n) \geq (\Phi_n)^{\xi}$ for $n$ large enough, and $\Phi_n\cinp\infty$. As a consequence, \[N_{\Gamma_n, \Phi_n}(\Upsilon_n)\cinp\infty \text{\ as\ } n\to\infty.\] Then by \lemmaref{lem:A}, both \eqref{111} and \eqref{222} converge to $0$ as $n\to\infty$. For \eqref{333}, on the one hand, $\min_{a\in\A}\K(\mu_{a} + \frac{\delta}{2}, \mu_{a} + \delta)$ is a constant; on the other hand,  $ N' _{\Gamma_n, \Phi_n}(\Upsilon_n) \geq N_{\Gamma_n, \Psi_n}(\Upsilon_n) \geq (\Phi_n)^{\xi}$ for sufficiently large $n$ and $(\Phi_n)^{\xi} \in \omega\left(\mathcal{F}(\Phi_n)\right)$, therefore 
    \begin{align*}
        \frac{\mathcal{F}(\Phi_n)}{ N' _{\Gamma_n, \Phi_n}(\Upsilon_n)}\to 0 \text{\ \ as $n\to\infty$.}
    \end{align*}
     Hence \eqref{333} is always $0$ for $n$ large enough. This concludes the proof.
\end{proof}
\begin{lemma}\label{lem:til}
    Let $(\Phi_n)_{n\geq 1}$ be a sequence of random variables such that $\Phi_n\cinp\infty$. Then for every player $p$ and every arm $a$, \[\tilde{\mu}_{p, \Phi_n}(a)\cinp\mu_a.\]
\end{lemma}
\begin{proof}
    First note that $\tilde{\mu}_{p, \Phi_n}(a) = \hat{\mu}_{p, N_{p, \Phi_n}(a)}(a)$. By \lemmaref{lem:A} it suffices to prove that $N_{p, \Phi_n}(a)\cinp\infty$. By the definition of convergence in probability, we only need to show that for every $N > 0$ and every $\epsilon > 0$, we can find an $N_*$ such that for any $n \geq N_*$, 
    \[\Pr\left(N_{p, \Phi_n}(a) < N\right) \leq \epsilon.\]
    For each $n \geq 1$, let $\Upsilon_n$ be the arm that player $p$ has pulled most frequently by round $\Phi_n$ and let $\Psi_n$ be the last round in the first $\Phi_n$ rounds that $p$ pulled $\Upsilon_n$. The definition of $\Psi_n$ implies that \[\Psi_n \geq \Phi_n \di K.\] 
    Therefore given $\Phi_n\cinp\infty$ we have  $\Psi_n\cinp\infty$. In addition, we have \[N_{p, \Psi_n}(\Upsilon_t) \geq \Phi_n \di K \geq \Psi_n \di K.\] Hence by \lemmaref{lem:B}, 
    \begin{align}
        \lim_{t\to\infty}\Pr\left(B^+_{p, \Psi_n}(\Upsilon_n) > \mu_{\Upsilon_n} + \epsilon\right)= 0\label{til-1}, \text{\ \ for any $\epsilon > 0$}.
    \end{align}
    Now let $\delta = (1 - \max_{a\in\A}\mu_a) \di 2$. Then for every $N > 0$ we have
    \begin{align*}
        \Pr\left(N_{p, \Phi_n}(a) < N\right) &= \Pr\left(N_{p, \Phi_n}(a) < N \wedge B^+_{p, \Psi_n}(\Upsilon_n) \geq B^+_{p, \Psi_n}(a)\right)\\
        &\leq \Pr\left(B^+_{p, \Psi_n}(\Upsilon_n) > 1 - \delta\right) + \Pr\left(\K'(\tilde{\mu}_{p, \Psi_n}(a), 1 - \delta) > \frac{\mathcal{F}(\Psi_n)}{N\cdot M}\right)\\
        &\leq \Pr\left(B^+_{p, \Psi_n}(\Upsilon_n) > \mu_{\Upsilon_n} + \delta\right) + \Pr\left(N\cdot M\cdot \K(0, 1 - \delta) > \mathcal{F}\left(\frac{\Phi_n}{K}\right)\right).
    \end{align*}
    For every $\epsilon > 0$, by \eqref{til-1} we can choose $N_0$ large enough such that for every $n \geq N_0$, $\Pr(B^+_{p, \Psi_n}(\Upsilon_n) > \mu_{\Upsilon_n} + \delta) < \epsilon \di 2$. Since $N\cdot M\cdot \K(0, 1 - \delta)$ is a constant, we can also choose $N_1$ large enough such that for every $n \geq N_1$, $N\cdot M\cdot \K(0, 1 - \delta) \leq \mathcal{F}( n\di K)$. Finally, according to the assumption that $\Phi_n\cinp\infty$, we can choose $N_2$ large enough such that for every $n \geq N_2$, $\Pr(\Phi_n < N_2) \leq \epsilon \di 2$. Thus, for every $n \geq N_* = \max(N_0, N_1, N_2)$, $\Pr\left(N_{p, \Phi_n}(a) < N\right) < \epsilon$.
\end{proof}
Then we can start our proof of \lemmaref{lem: greater}. In fact, it is equivalent to prove the following equation:
    \begin{align*}
        \lim_{n\to\infty}\Pr\left(N_{{\Phi_n}}(a) < \frac{(1 - \delta)\mathcal{F}({\Phi_n})}{{\K}(\nu_a, \nu_{a^*})}\right) = 0.
    \end{align*}
    For every round $t$, we define a random player $p_t$ such that 
    \begin{align}
         N' _{p_t, t}(a) \leq  N' _{p', t}(a), \text{\ \ for every $p'\in[M]$}\label{pt}
    \end{align}
    and we say player $p_t$ is \textit{chosen} for round $t$. For every $n$, we define a random player $\Gamma_n$ such that $\Gamma_n$ is the player who is chosen most frequently in the first $\Phi_n$ rounds. We also let $\Xi$ be a random set including the rounds (in the first $T$ rounds) in which player $\Gamma_n$ is chosen. We let random variable $\Upsilon_n$ be the arm that is pulled most frequently by $\Gamma_n$ in those rounds in $\Xi_n$ and let $\Psi_n$ be the last round in $\Xi_n$ that $\Gamma_n$ pulls $\Upsilon_n$. With all the definitions above, we have  
    \begin{align}
        N_{\Gamma_n, \Psi_n}(\Upsilon_n) \geq \frac{{\Phi_n}}{MK} \geq \frac{\Psi_n}{MK}.\label{con1}
    \end{align}
    as well as
    \begin{align}
        \Psi_n \geq \frac{{\Phi_n}}{MK}.\label{con2}
    \end{align} 
    Note that $\Phi_n\cinp\infty$ implies ${\Phi_n}\cinp\infty$, which by \eqref{con2} in turn implies
    \begin{align}
        \Psi_n\cinp\infty\label{con3}
    \end{align}
    For any fixed $T$, clearly we have 
\begin{align*}
    &N_{{\Phi_n}}(a) < \frac{(1 - \delta)\mathcal{F}({\Phi_n})}{{\K}(\nu_a, \nu_{a^*})} \implies\\ 
    &\hspace{5em}N_{t}(a) < \frac{(1 - \delta)\mathcal{F}({\Phi_n})}{{\K}(\nu_a, \nu_{a^*})}, \text{\ \ for every $t \leq {\Phi_n}$.}
\end{align*}
By \claimref{claimforgreater} and \eqref{pt}, 
    \begin{align*}
        &N_{{\Phi_n}}(a) < \frac{(1 - \delta)\mathcal{F}({\Phi_n})}{{\K}(\nu_a, \nu_{a^*})} \implies\\ 
        &\hspace{5em} N' _{p_t, t}(a) \leq \frac{(1 - \delta)\mathcal{F}({\Phi_n})}{{\K}(\nu_a, \nu_{a^*})}, \text{\ \ for every $t \leq {\Phi_n}$.}    
    \end{align*}
    Furthermore, by the definitions of random variables $\Gamma_n$, $\Psi_n$, and $\Upsilon_n$,
    \begin{align*}
        N_{{\Phi_n}}(a) < \frac{(1 - \delta)\mathcal{F}({\Phi_n})}{{\K}(\nu_a, \nu_{a^*})}\implies\left\{
            \begin{aligned}
                & N' _{\Gamma_n, \Psi_n}(a) < \frac{(1 - \delta)\mathcal{F}({\Phi_n})}{\K(\nu_a, \nu_{a^*})}\\
                &B^+_{\Gamma_n, \Psi_n}(\Upsilon_n) \geq B^+_{\Gamma_n, \Psi_n}(a)
            \end{aligned}
            \right.
    \end{align*}
    Hence we have
    \begin{align}
        &\lim_{n\to\infty}\Pr\left(N_{{\Phi_n}}(a) < \frac{(1 - \delta)\mathcal{F}({\Phi_n})}{{\K}(\nu_a, \nu_{a^*})}\right) \nonumber\\
        &\hspace{4em}\leq\lim_{n\to\infty}\Pr\left( N' _{\Gamma_n, \Psi_n}(a) < \frac{(1 - \delta)\mathcal{F}({\Phi_n})}{\K(\nu_a, \nu_{a^*})} \wedge B^+_{\Gamma_n, \Psi_n}(\Upsilon_n) \geq B^+_{\Gamma_n, \Psi_n}(a)\right)\nonumber\\
        &\hspace{7em}\leq\lim_{n\to\infty}\Pr\left(B^+_{\Gamma_n, \Psi_n}(\Upsilon_n) > \mu_{\Upsilon_n} + \epsilon\right)\label{11} \\
        &\hspace{7em}+ \sum_{p = 1}^M\lim_{n\to\infty}\Pr\left(\tilde{\mu}_{p, \Psi_n}(a) < \mu_a - \epsilon\right)\label{22}\\
        &\hspace{7em}+ \lim_{n\to\infty}\Pr\left(\K(\mu_a - \epsilon, \mu^* + \epsilon)\geq\frac{\mathcal{F}(\Psi_n)\K(\nu_a, \nu_{a^*})}{(1 - \delta)\mathcal{F}({\Phi_n})}\right),\label{33}
    \end{align}
    where $\epsilon$ is to be determined later. By \lemmaref{lem:B} and \eqref{con1}, the first term \eqref{11} is $0$.  
    By \lemmaref{lem:til} and \eqref{con3}, the second term \eqref{22} is $0$. For the third term \eqref{33}, by \eqref{con2} we have
    \begin{align}
    &\lim_{n\to\infty}\Pr\left(\K(\mu_a - \epsilon, \mu^* + \epsilon)\geq\frac{\mathcal{F}(\Psi_n)\K(\nu_a, \nu_{a^*})}{(1 - \delta)\mathcal{F}({\Phi_n})}\right)\nonumber\\
    &\hspace{3em}\leq  \lim_{n\to\infty}\Pr\left(\K(\mu_a - \epsilon, \mu^* + \epsilon)\geq\frac{\mathcal{F}\left(\frac{{\Phi_n}}{MK}\right)\K(\nu_a, \nu_{a^*})}{(1 - \delta)\mathcal{F}({\Phi_n})}\right)\nonumber
\end{align}
Now we decide $\epsilon$ to be a positive real number small enough such that there exists a $T_0$ satisfying
\begin{align*}
    \K(\mu_a - \epsilon, \mu^* + \epsilon) < \frac{\mathcal{F}(t \di (MK))\K(\nu_a, \nu_{a^*})}{(1 - \delta)\ln(t)}, \text{\ \ for every $t \geq T_0$}.
\end{align*}
Thus, 
    \begin{align*} 
        \lim_{n\to\infty}\Pr\left(\K(\mu_a - \epsilon, \mu^* + \epsilon) \geq \frac{\mathcal{F}\left({\Phi_n} \di (MK)\right)\K(\nu_a, \nu_{a^*})}{(1 - \delta)\ln(\Phi_t)}\right) \leq \lim_{n\to\infty}\Pr\left({\Phi_n} < T_0\right) = 0.
    \end{align*}
    This concludes the proof.

\section{Proof of \claimref{claimforgreater}}
    It suffices to prove $\sum_{p = 1}^M  N' _{p, t}(a) \leq M\cdot N_t(a)$. In fact, 
    \begin{align*} 
        \sum\nolimits_{p = 1}^M  N' _{p, t}(a) &\leq \sum\nolimits_{p = 1}^M \left(N_{p, t}(a) + (M - 1)\cdot\left(N_{p, t}(a) - N_{\ell(t)}(a)\right)\right)\text{\ \ \ \ \ (by Definition)}\\
                                           &\leq \sum\nolimits_{p = 1}^M \left(N_{\ell(t)}(a) + M\cdot\left(N_{p, t}(a) - N_{\ell(t)}(a)\right)\right)\\
                                           &\leq M\cdot N_{\ell(t)}(a) + M\cdot\sum\nolimits_{p = 1}^M (N_{p, t}(a) - N_{\ell(t)}(a)) \\[0.2em]
                                           &\leq M\cdot N_{\ell(t)}(a) + M\cdot (N_t(a) - N_{\ell(t)}(a))\\[0.7em]
                                           &=M\cdot N_t(a),
    \end{align*}
    which concludes the proof.

\section{Proof of \claimref{rel}}
    If $\alpha(\C) = 0$, then the right hand side becomes $M\cdot N_{p, t}(a)$. By definition,
    \begin{align*}
         N' _{p, t}(a) &\leq N_{p, t}(a) + (M - 1)\cdot \left(N_{p, t}(a) - N_{\ell(t)}(a)\right)\\
                            &\leq N_{p, t}(a) + (M - 1)\cdot N_{p, t}(a)\\
                            & = M\cdot N_{p, t}(a).
    \end{align*}
    If $N_{p, t}(a) = 0$, then $ N' _{p, t}(a) = 0$ and the bound is trivial. Now we assume $\alpha(\C) > 0$ and $N_{p, t}(a) > 0$. Note that
    \begin{align*}
        \frac{ N' _{p, t}(a)}{N_{p, t}(a)} = 1 + (M - 1)\cdot \min\left(1 - \frac{N_{\ell(t)}(a)}{N_{p, t}(a)}, \frac{u(\ell(t))}{N_{p, t}(a)}\right).
    \end{align*}
    Let $f(x) = \min\left(1 - N_{\ell(t)}(a)\di x, u(\ell(t))\di x\right)$. We have
    \begin{align*}
        \frac{ N' _{p, t}(a)}{N_{p, t}(a)} \leq 1 + (M - 1)\cdot\left(\sup_{x\in(0, \infty)}f(x)\right).
    \end{align*}
    Since $1 - N_{\ell(t)}(a)\di x$ is increasing in $(0, \infty)$ and $u(\ell(t))\di x$ is decreasing in $(0, \infty)$, $f(x)$ can be maximized if these two functions take the same value. In fact, when $x = x^* = N_{\ell(t)}(a) + u(\ell(t))$, we have $f(x^*) = 1 - N_{\ell(t)}(a)\di x^* = u(\ell(t))\di x^*$. Thus, 
    \begin{align*}
        \frac{ N' _{p, t}(a)}{N_{p, t}(a)} &\leq 1 + (M - 1)\cdot \frac{u(\ell(t))}{N_{\ell(t)}(a) + u(\ell(t))}\\
                                                &\leq 1 + (M - 1)\cdot \frac{\frac{N_{\ell(t)}(a)\di{\alpha}(\C) - N_{\ell(t)}(a)}{M}}{N_{\ell(t)}(a) + \frac{N_{\ell(t)}(a)\di{\alpha}(\C) - N_{\ell(t)}(a)}{M}}\\
                                                &=\frac{M}{1 + (M - 1)\alpha(\C)},
    \end{align*}
    which completes the proof.

\end{appendices}
\end{document}